\documentclass[11pt,letterpaper]{article}

\usepackage[utf8]{inputenc}
\usepackage[T1]{fontenc}
\usepackage[margin=1in]{geometry}

\usepackage[
  style=alphabetic,  
  maxcitenames=999,  
  maxalphanames=6,
  minalphanames=6,
  maxbibnames=999,
  giveninits=false,  
  sorting=nty,
  backref=true,
  uniquename=true,
  natbib=true,
  backend=bibtex
]{biblatex}

\let\cite\relax
\let\citep\relax
\let\citet\relax

\newbibmacro*{citet:group}{%
  \iffieldequals{namehash}{\citetlasthash}
  {\multicitedelim}
  {\ifnumequal{\value{citecount}}{1}
    {}
    {\bibclosebracket\multicitedelim}%
    \printtext[bibhyperref]{\printnames{labelname}}%
  \space\bibopenbracket}%
  \ifnumequal{\value{citecount}}{1}
  {\usebibmacro{prenote}}
  {}%
  \savefield{namehash}{\citetlasthash}%
}

\DeclareCiteCommand{\cite}[\mkbibbrackets]
{\usebibmacro{prenote}}
{%
  \printtext[bibhyperref]{\printfield{labelalpha}\printfield{extraalpha}}%
}
{\multicitedelim}
{\usebibmacro{postnote}}
\DeclareCiteCommand{\citep}[\mkbibbrackets]
{\usebibmacro{prenote}}
{%
  \printtext[bibhyperref]{\printfield{labelalpha}\printfield{extraalpha}}%
}
{\multicitedelim}
{\usebibmacro{postnote}}
\DeclareCiteCommand{\citet}
{\global\undef\citetlasthash}
{%
  \usebibmacro{citet:group}%
  \printtext[bibhyperref]{\printfield{labelalpha}\printfield{extraalpha}}%
}
{}
{\usebibmacro{postnote}\bibclosebracket}

\DeclareNameAlias{author}{given-family}

\DefineBibliographyStrings{english}{
  backrefpage = {cit\adddotspace on p\adddot},
  backrefpages = {cit\adddotspace on pp\adddot}
}

\DeclareBibliographyDriver{article}{%
  \printnames{author}%
  \setunit{\space}%
  \printtext{(\printfield{year})}%
  \newunit\newblock%
  \printfield{title}%
  \iffieldundef{title}{}{\addperiod\space}%
  \newunit\newblock%
  \printfield{journaltitle}%
  \setunit{\addcomma\space}\printfield{volume}%
  \setunit{\addcomma\space}\printfield{number}%
  \setunit{\addcomma\space}\printfield{pages}%
  \newunit\newblock%
  \usebibmacro{pageref}%
  \finentry%
}
\DeclareBibliographyDriver{inproceedings}{%
  \printnames{author}%
  \setunit{\space}%
  \printtext{(\printfield{year})}%
  \newunit\newblock%
  \printfield{title}%
  \iffieldundef{title}{}{\addperiod\space}%
  \newunit\newblock%
  In:\space\printfield{booktitle}%
  \setunit{\addcomma\space}\printlist{publisher}%
  \setunit{\addcomma\space}\printfield{pages}%
  \newunit\newblock%
  \usebibmacro{pageref}%
  \finentry%
}
\DeclareBibliographyDriver{book}{%
  \printnames{author}%
  \setunit{\space}%
  \printtext{(\printfield{year})}%
  \newunit\newblock%
  \printfield{title}%
  \iffieldundef{title}{}{\addperiod\space}%
  \newunit\newblock%
  \printlist{publisher}%
  \setunit{\addcomma\space}\printfield{location}%
  \newunit\newblock%
  \usebibmacro{pageref}%
  \finentry%
}
\DeclareBibliographyDriver{misc}{%
  \printnames{author}%
  \setunit{\space}%
  \printtext{(\printfield{year})}%
  \newunit\newblock%
  \printfield{title}%
  \iffieldundef{title}{}{\addperiod\space}%
  \newunit\newblock%
  \printfield{howpublished}%
  \newunit\newblock%
  \usebibmacro{pageref}%
  \finentry%
}

\usepackage{soul}  
\usepackage{microtype}  
\usepackage[american]{babel}  
\usepackage[style=american]{csquotes}  
\usepackage[table, dvipsnames, svgnames]{xcolor}  

\usepackage{amsthm}  
\usepackage{amsmath}  
\usepackage{amssymb}  
\usepackage{amsfonts}  
\usepackage{bm}  
\usepackage{mleftright}  

\usepackage{tikz}  
\usepackage{fancybox}  
\usepackage{graphicx}  
\usepackage{booktabs}  
\usepackage{makecell}  
\usepackage{multirow}  
\usepackage[normalsize]{subfigure}  

\usepackage{algorithmic}  
\usepackage{algorithm}  
\usepackage{enumitem}  

\usepackage{xspace}  

\usepackage{hyperref}  
\usepackage{cleveref}  

\usepackage{authblk}  

\newtheorem{thm}{Theorem}
\newtheorem{lem}{Lemma}
\newtheorem{cor}{Corollary}

\newtheorem{dfn}{Definition}

\newtheorem{asm}{Assumption}

\newenvironment{prf}[1][]
{
  
  \pushQED{\qed}
  \vspace{\medskipamount}
  \par\noindent
  \if\relax\detokenize{#1}\relax
  \emph{Proof.}
  \else
  \emph{#1.}
  \fi
  \enspace\ignorespaces
}
{
  \popQED\par
  \vspace{\medskipamount}
}

\newcommand{\replemname}{}
\newtheorem*{repleminner}{\replemname}

\AddToHook{env/tabular/begin}{\renewcommand{\arraystretch}{1.1}}
\AddToHook{env/tabular*/begin}{\renewcommand{\arraystretch}{1.1}}
\AddToHook{env/table/begin}{\setlength{\belowcaptionskip}{3pt}}
\AddToHook{env/table*/begin}{\setlength{\belowcaptionskip}{3pt}}

\algsetup{linenosize=\normalsize}

\hypersetup{colorlinks=true, linkcolor=blue, citecolor=red, urlcolor=orange}

\crefname{thm}{Theorem}{Theorems}
\crefname{lem}{Lemma}{Lemmas}
\crefname{cor}{Corollary}{Corollaries}
\crefname{prp}{Proposition}{Propositions}
\crefname{dfn}{Definition}{Definitions}
\crefname{cdt}{Condition}{Conditions}
\crefname{asm}{Assumption}{Assumptions}
\crefname{rmk}{Remark}{Remarks}
\crefname{pbl}{Problem}{Problems}
\crefname{stp}{Step}{Steps}
\crefname{prfskt}{Proof Sketch}{Proof Sketches}
\crefname{als}{Analysis}{Analysis}
\crefname{equation}{Eq.}{Eqs.}
\crefname{table}{Table}{Tables}
\crefname{section}{Section}{Sections}
\crefname{subsection}{Section}{Sections}
\crefname{subsubsection}{Section}{Sections}
\crefname{appendix}{Appendix}{Appendices}
\crefname{algorithm}{Algorithm}{Algorithms}
\crefname{figure}{Figure}{Figures}

\newcommand{\gr}{\nabla}

\newcommand{\transposed}{\mathsf{T}}
\newcommand{\trs}{\transposed}

\newcommand{\reals}{\mathbb{R}}

\renewcommand{\tilde}{\widetilde}

\DeclareMathOperator*{\argmax}{arg\,max}
\DeclareMathOperator*{\argmin}{arg\,min}

\DeclareMathOperator{\diam}{diam}

\newcommand{\abs}[1]{\left| #1 \right|}
\newcommand{\norm}[1]{\left\| #1 \right\|}

\newcommand{\norme}[1]{\norm{#1}_{2}}

\newcommand{\normop}[1]{\norm{#1}_{\operatorname{op}}}

\newcommand{\sbr}[1]{\left( #1 \right)}

\newcommand{\lbr}[1]{\left\{ #1 \right\}}

\makeatletter
\let\original@text@a\a
\let\original@text@b\b
\let\original@text@c\c
\let\original@text@d\d
\let\original@text@H\H
\let\original@text@i\i
\let\original@text@j\j
\let\original@text@k\k
\let\original@text@l\l
\let\original@text@L\L
\let\original@text@o\o
\let\original@text@O\O
\let\original@text@P\P
\let\original@text@r\r
\let\original@text@S\S
\let\original@text@t\t
\let\original@text@u\u
\let\original@text@v\v
\let\original@text@AA\AA
\let\original@text@SS\SS

\renewcommand{\a}{\TextOrMath{\original@text@a}{\mathbf{a}}}
\renewcommand{\b}{\TextOrMath{\original@text@b}{\mathbf{b}}}
\renewcommand{\c}{\TextOrMath{\original@text@c}{\mathbf{c}}}
\renewcommand{\d}{\TextOrMath{\original@text@d}{\mathbf{d}}}
\newcommand{\e}{\mathbf{e}}

\newcommand{\g}{\mathbf{g}}

\renewcommand{\i}{\TextOrMath{\original@text@i}{\mathbf{i}}}
\renewcommand{\j}{\TextOrMath{\original@text@j}{\mathbf{j}}}
\renewcommand{\k}{\TextOrMath{\original@text@k}{\mathbf{k}}}
\renewcommand{\l}{\TextOrMath{\original@text@l}{\mathbf{l}}}

\newcommand{\n}{\mathbf{n}}
\renewcommand{\o}{\TextOrMath{\original@text@o}{\mathbf{o}}}
\newcommand{\p}{\mathbf{p}}
\newcommand{\q}{\mathbf{q}}
\renewcommand{\r}{\TextOrMath{\original@text@r}{\mathbf{r}}}

\renewcommand{\t}{\TextOrMath{\original@text@t}{\mathbf{t}}}
\renewcommand{\u}{\TextOrMath{\original@text@u}{\mathbf{u}}}
\renewcommand{\v}{\TextOrMath{\original@text@v}{\mathbf{v}}}
\newcommand{\w}{\mathbf{w}}
\newcommand{\x}{\mathbf{x}}
\newcommand{\y}{\mathbf{y}}
\newcommand{\z}{\mathbf{z}}

\newcommand{\D}{\mathbf{D}}

\renewcommand{\H}{\TextOrMath{\original@text@H}{\mathbf{H}}}

\renewcommand{\L}{\TextOrMath{\original@text@L}{\mathbf{L}}}

\renewcommand{\O}{\TextOrMath{\original@text@O}{\mathbf{O}}}
\renewcommand{\P}{\TextOrMath{\original@text@P}{\mathbf{P}}}

\newcommand{\R}{\mathbf{R}}
\renewcommand{\S}{\TextOrMath{\original@text@S}{\mathbf{S}}}

\makeatother

\newcommand{\balpha}{\bm{\alpha}}

\newcommand{\bdelta}{\bm{\delta}}

\newcommand{\blambda}{\bm{\lambda}}

\newcommand{\bnu}{\bm{\nu}}

\newcommand{\eps}{\epsilon}
\newcommand{\veps}{\varepsilon}

\makeatletter
\renewcommand{\AA}{\TextOrMath{\original@text@AA}{\mathcal{A}}}
\makeatother
\newcommand{\BB}{\mathcal{B}}

\newcommand{\OO}{\mathcal{O}}

\newcommand{\QQ}{\mathcal{Q}}

\makeatletter
\renewcommand{\SS}{\TextOrMath{\original@text@SS}{\mathcal{S}}}
\makeatother

\newcommand{\XX}{\mathcal{X}}
\newcommand{\YY}{\mathcal{Y}}

\newcommand{\reg}{\normalfont\textrm{\textsc{Reg}}}
\newcommand{\ireg}{\normalfont\textrm{\textsc{I-Reg}}}
\newcommand{\dreg}{\normalfont\textrm{\textsc{D-Reg}}}
\newcommand{\idreg}{\normalfont\textrm{\textsc{I-D-Reg}}}
\newcommand{\gap}{\normalfont\textrm{\textsc{Err}}}

\def \Ot {\tilde{\mathcal{O}}}

\title{On the Relation Between Interval Regret and Dynamic Regret}

\author{Yi-Han Wang, \quad Peng Zhao, \quad Zhi-Hua Zhou}
\affil{
  State Key Laboratory for Novel Software Technology, Nanjing University, China\\
  School of Artificial Intelligence, Nanjing University, China\\
  \texttt{\{wangyh,zhaop,zhouzh\}@lamda.nju.edu.cn}
}
\date{}

\hypersetup{
  pdftitle={On the Relation Between Interval Regret and Dynamic Regret},
  pdfauthor={Yi-Han Wang, Peng Zhao, Zhi-Hua Zhou}
}

\begin{document}

\pagenumbering{roman}
\hypersetup{pageanchor=false}

\begin{titlepage}

  \maketitle
  \thispagestyle{plain}
  \normalsize

  \begin{center}
    \textbf{Abstract}
  \end{center}

  Non-stationary online learning has attracted much attention in recent years, where the classical notion of static regret is insufficient to guide algorithm design in changing environments.
  To address this limitation, \emph{interval regret} and \emph{dynamic regret} have been introduced as two representative performance metrics that strengthen static regret in complementary directions.
  Interval regret requires an online algorithm to achieve competitive static regret over every local time interval, whereas dynamic regret evaluates performance against an arbitrary sequence of time-varying comparators.
  Despite their importance, the relation between these two metrics has long remained unclear.
  In particular, prior work has often regarded interval regret as the stronger notion, based on the intuition that local guarantees should naturally induce global guarantees.
  Consequently, it is widely conjectured that an algorithm with optimal interval regret should automatically attain optimal dynamic regret as well, albeit with changing comparators.

  In this paper, we first establish a negative result that refutes this intuition of a metric-level implication.
  Specifically, we show that, for both convex functions and curved functions (including exp-concave and strongly convex functions), there exist instances in which an algorithm with optimal interval regret nevertheless fails to achieve optimal dynamic regret.
  We then show how to correctly leverage the local adaptivity to obtain optimal dynamic regret. In particular, we prove that optimal dynamic regret can be attained by invoking an interval regret minimization process over an enlarged Euclidean ball containing the original convex feasible domain and using a suitable domain-converted surrogate loss. This reduction applies for both convex and curved functions.
  As a byproduct, we obtain the first \emph{proper} and \emph{efficient} algorithm with \emph{optimal} dynamic regret for exp-concave functions, improving prior results while significantly simplifying the analysis.

\end{titlepage}

\setcounter{page}{2}
\tableofcontents

\clearpage

\pagenumbering{arabic}
\hypersetup{pageanchor=true}
\section{Introduction}
\label{sec:main-introduction}

Online Convex Optimization (OCO) is a fundamental and versatile framework for sequential decision-making with strong theoretical guarantees~\citep{book/Cambridge/cesa2006prediction,hazan2022introductionOCO}.
At each round $t\in[T]\triangleq\{1,\ldots,T\}$, a learner chooses $\x_t$ from a compact convex domain $\XX\subseteq\reals^d$, the environment reveals a convex loss $f_t:\XX\to\reals$, and the learner suffers $f_t(\x_t)$ and observes $\gr f_t(\x_t)$.
The classical performance metric is static regret, which compares the learner with a fixed comparator $\u\in\XX$,
\begin{equation*}
  \reg_T = \sum_{t=1}^T f_t(\x_t) - \sum_{t=1}^T f_t(\u).
\end{equation*}
Static regret minimization has led to a rich body of theory and a wide range of algorithmic developments over the last two decades~\citep{book22:FO-book}. However, in open and dynamic environments~\citep{NSR22:Zhou-OpenML}, competing against a single fixed comparator may be overly restrictive, making static regret insufficient for guiding algorithm design for non-stationary online learning.

\subsection{Interval Regret and Dynamic Regret}

To address the limitation, two complementary notions, namely \emph{interval regret} (also known as strongly adaptive regret)~\citep{journal07:Hazan-adaptive,ICML09:Hazan-adaptive,ICML15:Daniely-adaptive,AISTATS17:coin-betting-adaptive} and \emph{dynamic regret}~\citep{ICML03:zinkvich,ICML'13:dynamic-model,NIPS18:Zhang-Ader,NeurIPS'20:sword}, have been introduced to strengthen static regret in different directions and to guide the algorithm design for non-stationary online learning.

Specifically, interval regret evaluates local performance by comparing the learner against a fixed comparator $\z_I\in\XX$, which may depend on the interval $I$, on every interval $I\subseteq [T]$:
\begin{equation*}
  \ireg_I(\z_I) = \sum_{t\in I} f_t(\x_t)-\sum_{t\in I}  f_t(\z_I).
\end{equation*}
For convex functions, an optimal $\Ot(\sqrt{|I|})$ interval regret can be achieved~\citep{ICML15:Daniely-adaptive,AISTATS17:coin-betting-adaptive}; for exp-concave and strongly convex losses, the respective optimal rates are $\Ot(d\log|I|)$ and $\Ot(\log|I|)$~\citep{journal07:Hazan-adaptive,ICML09:Hazan-adaptive,ICML18:zhang-dynamic-adaptive}. Here, we use $\Ot(\cdot)$ notation to hide logarithmic factors in $T$ while retaining the $\log|I|$ dependence for clarity.
All of the above bounds match the optimal static regret rates on each interval up to logarithmic factors, simultaneously for every interval $I\subseteq [T]$. Subsequent work develops adaptive interval regret guarantees that depend on finer-grained problem-dependent quantities~\citep{ICML19:Zhang-Adaptive-Smooth,arXiv26:gradient-variation-interval}.

On the other hand, dynamic regret evaluates global performance against an arbitrary time-varying comparator sequence $\u_1,\dots,\u_T \in \XX$:
\begin{equation}
  \label{eq:main-dynamic-regret-intro}
  \dreg_T(\lbr{\u_t}_{t=1}^T) = \sum_{t=1}^T f_t(\x_t)- \sum_{t=1}^T f_t(\u_t),
\end{equation}
where the comparator sequence is typically assumed to have sublinear-in-$T$ path length, $P_T=\sum_{t=2}^T \norme{\u_t-\u_{t-1}}$, in order to obtain meaningful guarantees; this quantity measures the degree of non-stationarity. For convex functions, \citet{NIPS18:Zhang-Ader} established the minimax-optimal $\OO(\sqrt{T(1+P_T)})$ rate without  prior knowledge of $P_T$.
For curved losses (namely exp-concave or strongly convex functions), \citet{COLT21:baby-exp-concave,AISTATS22:sc-proper} obtained the minimax-optimal $\Ot(1+T^{1/3}P_T^{2/3})$ rate, with dimension-dependent factors (omitted here) and additional requirements on the feasible domain, as discussed later.
This line of research also yields refined bounds that adapt to finer problem-dependent quantities~\citep{NeurIPS'20:sword,JMLR24:Sword++official}, which play an important role in studying more complicated time-varying games~\citep{ICML'22:TVgame} and MDP settings~\citep{ICML'22:mdp}.

Despite this progress, the two metrics are generally incomparable, and the two lines of research have largely developed separately.
Some works aim to obtain the best of both worlds by studying \emph{interval dynamic regret}, which requires the online algorithm to be competitive with time-varying comparators over every interval~\citep{AISTATS20:Zhang,ICML20:Ashok,JMLR'25:efficient,arXiv26:gradient-variation-interval}.
Nevertheless, such joint guarantees still leave open the fundamental relation between interval regret and dynamic regret.

\subsection{Metric-level Implication}

Interval regret provides a \emph{local} guarantee on every time interval, whereas dynamic regret measures \emph{global} performance against a changing comparator sequence.
This naturally suggests a local-to-global intuition: learning well on every interval might suffice to track changing comparators over the full horizon; that is, OIR (optimal interval regret) may automatically imply ODR (optimal dynamic regret).
This intuition is further reinforced by the reduction of \citet{ICML18:zhang-dynamic-adaptive} and its textbook presentation by \citet{hazan2022introductionOCO}, where interval regret is described as a more general metric than dynamic regret: ``This metric can be shown to be more general than dynamic regret, in the sense that the bounds we prove also imply low dynamic regret''.

However, we argue that the \emph{metric-level implication} from OIR to ODR is considerably more subtle.
In fact, the reduction of \citet{ICML18:zhang-dynamic-adaptive} only establishes a metric-level implication from interval regret to a special form of the general dynamic regret~\eqref{eq:main-dynamic-regret-intro}, namely the \emph{worst-case dynamic regret}, defined as
\begin{equation}
  \label{eq:worst-case-dynamic-regret}
  \dreg_{T}^{\text{worst}} = \sum_{t=1}^T f_t(\x_t) - \sum_{t=1}^T f_t(\x_t^*), \qquad
  \x_t^*\in\argmin_{\x\in\XX}f_t(\x).
\end{equation}
They show that $\Ot(\sqrt{|I|})$ interval regret implies $\Ot(T^{2/3}V_T^{1/3})$ worst-case dynamic regret, where $V_T = \sum_{t=2}^T \max_{\x \in \XX}|f_t(\x)-f_{t-1}(\x)|$ measures function variation.
This reduction concerns worst-case dynamic regret in~\eqref{eq:worst-case-dynamic-regret}, rather than general dynamic regret against arbitrary comparator sequences in~\eqref{eq:main-dynamic-regret-intro}.
For the latter, their corresponding reduction yields only $\Ot(T^{2/3}P_T^{1/3})$ \citep{AISTATS20:Zhang}, leaving a substantial gap to the $\Omega(\sqrt{TP_T})$ minimax lower bound.
Moreover, worst-case dynamic regret can encourage \emph{overfitting} to individual losses, which may contain sampling noise \citep{NIPS18:Zhang-Ader,ICML'24:wavelet}.

Indeed, there are two more relevant results concerning general dynamic regret, both established in special online learning settings.
For prediction with expert advice (namely, online linear optimization over the simplex), \citet{COLT15:Luo-AdaNormalHedge} show that an algorithm with OIR directly enjoys ODR.
For curved losses over box-type feasible domains, \citet{COLT21:baby-exp-concave} show that OIR implies ODR, via an intricate KKT-based analysis that relies crucially on the box-domain structure.
These special cases motivate a fundamental question about the relation between the two metrics that capture the local and global aspects of non-stationary online learning:

\begin{center}
  \begingroup
  \setlength{\fboxsep}{10pt}
  \shadowbox{%
    \begin{minipage}{\dimexpr\linewidth-2\fboxsep-2\fboxrule-\shadowsize\relax}
      Is there any \emph{metric-level implication} from interval regret to dynamic regret?

      Namely, if an online learning algorithm is equipped with optimal interval regret guarantees, does it automatically enjoy optimal dynamic regret?
    \end{minipage}%
  }
  \endgroup
\end{center}

\subsection{Our Contributions}
In this paper, we provide an in-depth investigation of the relation between interval regret and dynamic regret. We establish both negative and positive results: on the negative side, we refute the long-standing intuition of a metric-level implication from interval regret to dynamic regret; on the positive side, we further show how to leverage the local adaptivity of interval regret minimization to attain favorable dynamic regret bounds.

\paragraph{Contribution 1: Refuting the metric-level implication from OIR to ODR.}
Revisiting the existing reductions reveals the importance of domain geometry.
The reduction of \citet{ICML18:zhang-dynamic-adaptive} concerns worst-case dynamic regret and therefore does not establish the OIR-to-ODR implication for arbitrary comparator sequences considered here.
The Luo--Schapire and Baby--Wang results do address general dynamic regret, but exploit specific geometric structures: the simplex for linear losses and box domains for curved losses, respectively \citep{COLT15:Luo-AdaNormalHedge,COLT21:baby-exp-concave,AISTATS22:sc-proper}.
We show that the metric-level implication fails in general, with domain geometry determining when it can hold.
Specifically, we construct hard instances with linear losses and quadratic losses, whose decision sequences satisfy OIR yet incur dynamic regret $\Omega(\tau^{2/5}T^{3/5})$ and $\Omega(\tau^{3/5}T^{2/5})$, respectively, under a path-length budget $\tau$.
Both rates are polynomially worse than the corresponding optimal rates.
Conversely, suitable geometry allows the Luo--Schapire implication to extend beyond the simplex.
For linear losses, we extend it to polytopes, with constants depending on their geometry, and to sufficiently small Euclidean smoothings of these polytopes.

\begin{table}[!t]
  \centering
  \small
  \setlength{\tabcolsep}{4pt}
  \renewcommand{\arraystretch}{1.4}
  \caption{OIR requirements on the surrogate losses and resulting ODR guarantees on the original losses. The ODR column bounds $\dreg_T(\lbr{\u_t}_{t=1}^T)$.}
  \label{tab:intro-reduction-rates}
  \begin{tabular}{@{}cccc@{}}
    \toprule
    Loss class & OIR on $h_t$ & Surrogate loss & ODR on $f_t$ \\
    \midrule
    Convex & $\Ot(\sqrt{|I|})$ & \cref{eq:surrogate-convex} &
    $\Ot(\sqrt{T(1+P_T)})$ \\[2pt]
    Exp-concave & $\Ot(d\log|I|)$ & \cref{eq:surrogate-exp-concave} &
    $\Ot\big(d+d^{\frac{2}{3}}P_T^{\frac{2}{3}}T^{\frac{1}{3}}\big)$  \\[2pt]
    Strongly convex & $\Ot(\log|I|)$ & \cref{eq:surrogate-strongly-convex} &
    $\Ot\big(1+P_T^{\frac{2}{3}}T^{\frac{1}{3}}\big)$  \\
    \bottomrule
  \end{tabular}
\end{table}

\paragraph{Contribution 2: Achieving ODR by leveraging local adaptivity of OIR.}
Despite the absence of a general metric-level implication, we propose a wrapper that takes an algorithm with OIR on a Euclidean ball and produces a (different) algorithm with ODR on the original losses over any compact convex domain.
The wrapper uses surrogate losses on $\YY=\BB(D_\XX)$, with twice the radius of the ball enclosing $\XX\subseteq\BB(D_\XX/2)$.
Its decisions remain in $\XX$, and for every comparator path it guarantees the dynamic regret bounds in \cref{tab:intro-reduction-rates}.
The dependence on $T$ and $P_T$ is optimal up to logarithmic factors \citep{NIPS18:Zhang-Ader,COLT21:baby-exp-concave}.
The key insight is that OIR guarantees on the enlarged domain and the corresponding surrogate losses provide additional information beyond the OIR guarantees shown insufficient by our negative result.
Notably, the final algorithm for minimizing dynamic regret is \emph{proper} in the sense that the decision at each round is guaranteed to lie within the feasible domain $\XX$, owing to the domain conversion technique~\citep{COLT18:black-box-reduction,ICML20:Ashok,COLT19:Lipschitz-MetaGrad,COLT23:OQNS}.

As a byproduct, we have obtained a \emph{proper} and \emph{efficient} algorithm with \emph{optimal} dynamic regret for exp-concave losses without requiring $\XX$ to be a box.
To our knowledge, this is the first algorithm to combine properness, efficiency, and optimal dynamic regret for exp-concave online learning on general convex domains. Importantly, our regret analysis is substantially simpler than the original KKT analysis specialized to box feasible domains~\citep{COLT21:baby-exp-concave}.

\paragraph{Notations.}
We write $\Pi_\XX[\y] = \argmin_{\x\in\XX} \norme{\y-\x}$ for the Euclidean projection from $\y$ onto $\XX$ and $\BB(R) = \{ \x \in \reals^d : \norme{\x} \le R \}$ for the Euclidean ball with radius $R$ centered at the origin.

\paragraph{Organization.}
\Cref{sec:main-geometry} establishes two complementary results on the metric-level implication from OIR to ODR.
\Cref{sec:main-reduction} shows how to leverage the local adaptivity of interval regret to attain optimal dynamic regret.
\Cref{sec:main-discussion} concludes the paper and discusses future directions.
Omitted proofs and details are deferred to \cref{sec:cert-dichotomy}--\ref{sec:algorithm-details}.

\section{Negative result: metric-level implication from OIR to ODR}
\label{sec:main-geometry}

In this section, we first prove that OIR guarantees alone do not imply ODR on general convex domains. On the other hand, we further extend the Luo--Schapire argument \citep{COLT15:Luo-AdaNormalHedge}, which gives a direct reduction from OIR to ODR for linear losses, from the simplex to polytopes and their sufficiently small Euclidean smoothings.
Together, these results reveal that OIR is not necessarily stronger than ODR in general, thereby refuting the long-standing intuition, while also highlighting the crucial role of domain geometry in the relation between the two performance measures.

\subsection{Main results}
\label{sec:geometry-main-results}

We study what can be deduced from OIR guarantees.
Throughout, we use only the OIR guarantees satisfied by a decision sequence, without further information about the algorithm generating those decisions.
Our negative result applies to both convex and curved losses in $\reals^d$.

\begin{thm}[OIR does not imply ODR on general convex domains]
  \label{thm:geometry-separation}
  For every sufficiently large horizon $T$ and path-length budget $\tau\ge1$ and $\tau=o(T)$:
  \begin{enumerate}[label=(\roman*)]
    \item There exist $1$-Lipschitz linear losses on the unit $\ell_2$ ball, a comparator path with $P_T\le \tau$, and decisions satisfying $\ireg_I(\z_I)\le\sqrt{|I|}$ for every interval $I\subseteq[T]$ and $\z_I\in\XX$, while
      \begin{equation*}
        \dreg_T(\lbr{\u_t}_{t=1}^T) = \Omega\Big(\tau^{\frac{2}{5}}T^{\frac{3}{5}}\Big).
      \end{equation*}
    \item There exist $3$-Lipschitz, $1$-strongly convex, and $1/9$-exp-concave quadratic losses on the unit $\ell_{3/2}$ ball, a comparator path with $P_T\le \tau$, and decisions satisfying $\ireg_I(\z_I)\le1$ for every interval $I\subseteq[T]$ and $\z_I\in\XX$, while
      \begin{equation*}
        \dreg_T(\lbr{\u_t}_{t=1}^T) = \Omega\Big(\tau^{\frac{3}{5}}T^{\frac{2}{5}}\Big).
      \end{equation*}
  \end{enumerate}
\end{thm}

\subsection{Key analyses}
\label{sec:geometry-key-analyses}

Fix a time horizon $T$, continuous convex losses $\lbr{f_t}_{t=1}^T$ on a compact convex domain $\XX$, and a comparator path $\lbr{\u_t}_{t=1}^T$.
Let the order function $\rho:[T]\to\reals_+$ be nondecreasing and the coefficient $A_T\ge1$.
We consider the OIR guarantees on all intervals $I\subseteq[T]$ and all comparators $\z_I\in\XX$:
\begin{equation}
  \label{eq:oir-guarantees}
  \ireg_I(\z_I)\le A_T \cdot \rho(|I|),
\end{equation}
with $d$ and $\log T$ factors absorbed into $A_T$.
The optimal interval scales are $\rho(n)=\sqrt n$ for convex losses and $\rho(n)=1$ for exp-concave and strongly convex losses.
For an interval comparator $\z_I\in\XX$ on each interval, define the interval-to-dynamic approximation error by
\begin{equation}
  \label{eq:main-i2d-gap}
  \gap_I(\z_I) \triangleq \sum_{t\in I} f_t(\z_I) - \sum_{t\in I} f_t(\u_t).
\end{equation}
Assign nonnegative weights $\lambda_I$ to intervals so that each round has total weight one.
Aggregating the interval guarantees gives
\begin{equation}
  \label{eq:main-cover-costs}
  \begin{gathered}
    \dreg_T = \sum_I \lambda_I \sbr{\ireg_I(\z_I) + \gap_I(\z_I)} \le A_T R_{T,\rho}+\gap_T, \\
    R_{T,\rho} \triangleq \sum_I \lambda_I \rho(|I|), \qquad
    \gap_T \triangleq \sum_I \lambda_I \gap_I(\z_I).
  \end{gathered}
\end{equation}
We call $R_{T,\rho}$ the \emph{interval complexity} and $\gap_T$ the total \emph{approximation error}.
When the loss class is clear, we write $R_T$ for $R_{T,\rho}$.
Such a weighting is called a \emph{homogeneous aggregation}; formally, the feasible weights form
\begin{equation}
  \label{eq:main-exact-cover}
  \Lambda_T=\lbr{
    \lbr{\lambda_I}_{I=[a,b]\subseteq [T]} : ~~ \lambda_I\ge0, ~~ \sum_{I:t\in I}\lambda_I=1\quad(t\in[T])
  }.
\end{equation}
By definition, $\sum_I\lambda_I|I|=T$.
The following duality shows that such aggregation characterizes the worst possible dynamic regret deduced from the OIR guarantees.
Thus, a lower bound built on homogeneous aggregation represents every proof using only this information.

\begin{thm}[homogeneous aggregation represents all proofs]
  \label{thm:cover-duality}
  With preceding notations, let
  \begin{gather*}
    m_t = \min_{\z\in\XX} f_t(\z), ~~
    M_t = \max_{\z\in\XX} f_t(\z), ~~
    b_I = \min_{\z\in\XX} \sum_{t\in I} f_t(\z), ~~
    c_I =
    \begin{cases}
      b_I + A_T \rho(|I|), &|I|>1, \\
      \min\{ m_t+A_T\rho(1), M_t \}, &I=\{t\}.
    \end{cases}
  \end{gather*}
  Then both extrema in the following equality are attained:
  \begin{equation}
    \label{eq:main-cover-duality}
    \max_{\substack{
        \lbr{\x_t}_{t=1}^T\in\XX^T \text{ satisfying \cref{eq:oir-guarantees}}
    }} \dreg_T(\lbr{\u_t}_{t=1}^T)
    = \min_{\blambda\in\Lambda_T} \sum_I \lambda_I c_I - \sum_{t=1}^T f_t(\u_t).
  \end{equation}
\end{thm}

See \cref{sec:cert-lp-completeness} for details.
Indeed, the OIR guarantees impose linear inequalities on the scalar learner losses $q_t=f_t(\x_t)$.
Linear programming duality then gives \cref{thm:cover-duality}.
The singleton correction in $c_I$ enforces the attainable loss range, so the LP solution $\lbr{q^\star_t}_{t=1}^T$ is attainable by decisions in $\XX$.
In the constructions of \cref{thm:geometry-separation}, $A_T=1$, $\rho(1)=1$, $m_t=0$, and $M_t\ge1$, so this correction vanishes.
For $A_T\ge1$, we still have $b_I+A_T\rho(|I|)-c_I \le (A_T-1)\rho(|I|)$.
Thus
\begin{equation*}
  \sum_I \lambda_I c_I - \sum_{t=1}^T f_t(\u_t) \ge A_T R_{T,\rho} + \gap_T - (A_T-1) R_{T,\rho} = R_{T,\rho} + \gap_T.
\end{equation*}
The same lower bounds apply to the right-hand side of \cref{eq:main-cover-duality}, so the counterexamples still hold.

\paragraph{The dilemma between interval complexity and approximation error.}
The constructions in \cref{thm:geometry-separation} force a tradeoff between interval complexity and approximation error.
Short intervals incur large interval complexity, while long intervals entail large approximation error.
Here we decompose the horizon into $B$ blocks of length $L=T/B$.
For linear losses on the $\ell_2$ ball, the dilemma is
\begin{equation*}
  R_T = \Omega\sbr{\sqrt{TB}} \quad \text{or} \quad
  \gap_T = \Omega\sbr{T(\tau/B)^2}.
\end{equation*}
Choosing $B=\Theta(\tau^{4/5}T^{1/5})$ gives the lower bound $\Omega(\tau^{2/5}T^{3/5})$.
For exp-concave and strongly convex losses on the $\ell_{3/2}$ ball, the dilemma is
\begin{equation*}
  R_T = \Omega(B) \quad \text{or} \quad
  \gap_T = \Omega\sbr{T(\tau/B)^{\frac{3}{2}}}.
\end{equation*}
Choosing $B=\Theta(\tau^{3/5}T^{2/5})$ gives the lower bound $\Omega(\tau^{3/5}T^{2/5})$.

\subsection{Extending Luo--Schapire arguments to polytopes and their smoothings}
\label{sec:main-polyhedral-extensions}

On the other hand, polyhedral geometry allows OIR to imply ODR for linear losses.
We extend the Luo--Schapire argument from the simplex to polytopes and their sufficiently small Euclidean smoothings \citep{COLT15:Luo-AdaNormalHedge}.
These results give an interpretation of the geometric separation:
on the smooth domains in \cref{thm:geometry-separation}, replacing moving comparators $\u_t$ with per-interval static comparators $\z_I$ incurs an unavoidable dilemma between interval complexity and approximation error, which polyhedral representations eliminate for linear losses.
Detailed discussions appear in \cref{sec:cert-luo-schapire}.

The key is to represent each comparator on a polytope $\QQ$ by probability weights on its vertices, with a Lipschitz coefficient $G_\QQ$ controlling how these weights vary along the comparator path.
\Cref{thm:cert-polyhedral-reduction} shows that the OIR guarantees in \cref{eq:oir-guarantees} with $\rho(|I|)=\sqrt{|I|}$ imply
\begin{equation*}
  R_T\le\sqrt{T\sbr{1+\frac{G_\QQ}{2}P_T}}, \qquad
  \gap_T=0, \qquad
  \dreg_T\le A_T\sqrt{T\sbr{1+\frac{G_\QQ}{2}P_T}}.
\end{equation*}
For the smoothed domain $\XX_\eps=\QQ+\eps\BB$, where $\BB$ is the Euclidean unit ball, \cref{thm:cert-smoothed-polygon} shows that the same OIR guarantees imply
\begin{equation*}
  R_T\le\sqrt{T\sbr{1+\frac{G_\QQ}{2}P_T}}, \qquad
  \gap_T\le\eps T, \qquad
  \dreg_T\le A_T\sqrt{T\sbr{1+\frac{G_\QQ}{2}P_T}}+\eps T.
\end{equation*}
The constant $G_\QQ$ depends only on the geometry of $\QQ$ and is $\OO(1)$ for standard simplices and $\ell_\infty$ balls in fixed dimension, but can become large enough to render the dynamic-regret bound vacuous for ill-conditioned polytopes such as regular $T$-gons, consistent with our geometric insight that polytopes approaching smooth convex bodies can inherit the obstruction to the OIR-to-ODR implication.

\begin{algorithm}[!t]
  \caption{OIR-algorithm-to-ODR-algorithm wrapper}
  \label{alg:main-reduction}
  \begin{algorithmic}
    \REQUIRE Domain and losses satisfying the applicable part of \cref{asm:domain-gradient,asm:curvature}, parameters from \cref{tab:loss-parameters}, and one algorithm $\AA$ with OIR on $\YY=\BB(D_\XX)$ as in \cref{thm:algorithm-reduction}.
    \ENSURE Decisions $\lbr{\x_t}_{t=1}^T \in \XX^T$.
    \STATE Initialize $\AA$ on $\YY$ with parameters in the right column of \cref{tab:loss-parameters}.
    \FOR{$t = 1, \ldots, T$}
    \STATE Obtain $\y_t \in \YY$ from $\AA$, play
    $\x_t = \Pi_\XX[\y_t]$, and observe $\g_t = \gr f_t(\x_t)$.
    \STATE Set $\n_t = \y_t - \x_t$ and $\d_t = \g_t + \dfrac{\max\{ -\g_t^\trs \n_t, 0 \}}{\norme{\n_t}^2} \n_t$, with the convention $\d_t = \g_t$ for $\n_t = \bm0$.
    \STATE Feed $\AA$ the full loss $h_t$ defined in \cref{eq:surrogate-convex,eq:surrogate-exp-concave,eq:surrogate-strongly-convex} for the selected loss class.
    \ENDFOR
  \end{algorithmic}
\end{algorithm}

\begin{table}[th]
  \centering
  \caption{Domain and loss parameters before and after domain conversion.}
  \label{tab:loss-parameters}
  \begin{tabular}{lll}
    \toprule
    Parameter & $f_t$ on $\XX$ & $h_t$ on $\YY$ \\
    \midrule
    Enclosing-ball diameter & $D_\XX$ & $D_\YY = 2D_\XX$ \\
    Lipschitz coefficient & $G$ &
    $G_h =
    \begin{cases}
      G, &\text{convex},\\
      (1 + \gamma D_\YY G)G, &\text{exp-concave},\\
      G + \lambda(D_\XX+D_\YY)/2, &\text{strongly convex}.
    \end{cases}$ \\
    Exp-concavity coefficient & $\alpha$ &
    $\alpha_h = \gamma/(1 + \gamma D_\YY G)^2$ \\
    Exp-concave curvature & $\gamma$ in \cref{eq:main-gamma} &
    $\gamma_h = \alpha_h/2$ \\
    Strong-convexity coefficient & $\lambda$ & $\lambda_h = \lambda$ \\
    Smoothness coefficient & N/A &
    $H_h = 2G_h/D_\XX$, \emph{only in hindsight analyses}. \\
    \bottomrule
  \end{tabular}
\end{table}

\section{Positive result: leveraging local adaptivity to achieve ODR}
\label{sec:main-reduction}

Although \cref{sec:main-geometry} refutes a general metric-level implication from OIR to ODR, the algorithmic ability to minimize regret locally in time can nonetheless be used to globally track changing comparators.
In this section, we show how to correctly leverage the local adaptivity of an OIR algorithm to achieve a (different) algorithm with ODR.
The key ingredients are domain conversion and the construction of corresponding surrogate losses.
This goes beyond directly applying an OIR algorithm to the original problem, whose interval-regret guarantee alone is insufficient, as shown in \cref{sec:main-geometry}.

In \cref{alg:main-reduction}, we propose a wrapper that takes an OIR algorithm for surrogate losses and produces an ODR algorithm for the original losses.
The loss class and parameters in \cref{tab:loss-parameters} are supplied at initialization.
The horizon $T$ may be known, but the path length $P_T$ is always unknown.

\subsection{Main result}
\label{sec:reduction-main-result}

\begin{asm}[bounded domain and bounded gradient]
  \label{asm:domain-gradient}
  The domain is convex and compact, and satisfies $\XX \subseteq \BB(D_\XX/2)$ for a fixed $D_\XX>0$.
  For a fixed $G>0$, each loss $f_t$ is differentiable and convex on $\XX$, with $\norme{\gr f_t(\x)}\le G$ for every $\x\in\XX$.
\end{asm}

\begin{asm}[curvature, when applicable]
  \label{asm:curvature}
  In the exp-concave regime, $\x \mapsto \exp(-\alpha f_t(\x))$ is concave on $\XX$ for a fixed $\alpha>0$.
  In the strongly convex regime, every pair $(\x,\u)\in\XX^2$ satisfies $f_t(\u)\ge f_t(\x)+\gr f_t(\x)^\trs(\u-\x)+\frac{\lambda}{2}\norme{\u-\x}^2$ for a fixed $\lambda>0$.
\end{asm}

Let $\y_t$ be the decision of $\AA$, and let $\x_t$ and $\d_t$ be its projection and the corrected gradient defined in \cref{alg:main-reduction}. The wrapper feeds $\AA$ the following surrogate loss, with $\gamma$ specified in \cref{eq:main-gamma}:
\begin{align}
  \label{eq:surrogate-convex}
  h_t(\y) &= \d_t^\trs \y
  && \text{(convex)}, \\
  \label{eq:surrogate-exp-concave}
  h_t(\y) &= \d_t^\trs \y + \frac{\gamma}{2} [\d_t^\trs (\y - \y_t)]^2
  && \text{(exp-concave)}, \\
  \label{eq:surrogate-strongly-convex}
  h_t(\y) &= \d_t^\trs \y + \frac{\lambda}{2} \norme{\y - \x_t}^2
  && \text{(strongly convex)}.
\end{align}

\begin{thm}[algorithmic OIR-to-ODR reduction]
  \label{thm:algorithm-reduction}
  Suppose \cref{asm:domain-gradient,asm:curvature} hold for the selected loss class.
  Suppose an OIR algorithm $\AA$ guarantees interval regret of order $\Ot(\sqrt{|I|})$, $\Ot(d\log|I|)$, or $\Ot(\log|I|)$ for convex, exp-concave, or strongly convex losses, respectively.
  Then the wrapper in \cref{alg:main-reduction}, using the OIR algorithm $\AA$ over the surrogate losses $\{h_t\}_{t=1}^T$ and the enlarged domain $\YY$, produces decisions $\x_t \in \XX$ and achieves the following ODR guarantees:
  \begin{align}
    \label{eq:dynamic-regret-bounds-convex}
    \sum_{t=1}^T f_t(\x_t) - \sum_{t=1}^T f_t(\u_t) &\leq \Ot(\sqrt{T(1+P_T)})
    && \text{(convex)}, \\
    \label{eq:dynamic-regret-bounds-exp-concave}
    \sum_{t=1}^T f_t(\x_t) - \sum_{t=1}^T f_t(\u_t) &\leq \Ot\big(d+d^{\frac{2}{3}}P_T^{\frac{2}{3}}T^{\frac{1}{3}}\big)
    && \text{(exp-concave)}, \\
    \label{eq:dynamic-regret-bounds-strongly-convex}
    \sum_{t=1}^T f_t(\x_t) - \sum_{t=1}^T f_t(\u_t) &\leq \Ot\big(1+P_T^{\frac{2}{3}}T^{\frac{1}{3}}\big)
    && \text{(strongly convex)}.
  \end{align}
  The dynamic regret bounds hold for any comparator sequence $\u_1,\u_2,\ldots,\u_T \in \XX$.
\end{thm}

The dependence on $T$ and $P_T$ is optimal up to logarithmic factors.
Established OIR algorithms satisfy the required interface \citep{ICML15:Daniely-adaptive,AISTATS17:coin-betting-adaptive,ICML18:zhang-dynamic-adaptive,NIPS21:dual-adaptive}.
As a byproduct, the reduction yields a proper and efficient algorithm with optimal dynamic regret for exp-concave losses on arbitrary compact convex domains, resolving a question left open by \citet{AISTATS22:sc-proper}.
Their analysis relies on a delicate coordinatewise treatment of KKT conditions specific to box domains; ours is substantially simpler even when $\XX$ is a box.
The mixability approach of \citet{ICML25:non-stationary-mixability} also provides proper guarantees beyond boxes, but faces nontrivial computational challenges.

\subsection{Key analyses}
\label{sec:reduction-key-analyses}

The following analysis presents the three key ingredients underlying the wrapper, with complete proofs in \cref{sec:algorithm-details}.

\subsubsection{Domain conversion safeguards the regret upper bound}
\label{sec:main-domain-conversion}

The geometric obstruction in \cref{thm:geometry-separation} motivates changing domains.
Specifically, by evaluating OIR on an enlarged domain, our method allows interval comparators $\z_I$ to lie outside $\XX$, while dynamic comparators $\u_t$ remain in $\XX$.
Consider a linear loss $f_t(\x) = \g_t^\trs \x$ and suppose the approximation error per round is $\eps$.
Simply pushing $\z_I$ to $\z_I'=\z_I-(\eps/\norme{\g_t}^2)\g_t$ reduces the approximation error $f_t(\z_I)-f_t(\u_t)$ by $\eps$ and reduces the approximation error to zero, which we call the improper learning benefit.

The problem is that, after such a push, $\z_I'$ may lie outside $\XX$, where the original loss $f_t$ is undefined.
Luckily, the established domain conversion techniques readily tackle this issue.
\citet{COLT18:black-box-reduction,ICML20:Ashok} first propose domain conversion for convex losses, and \citet{COLT19:Lipschitz-MetaGrad,COLT23:OQNS} extend the techniques to exp-concave losses.
The techniques show great value in improving the computational efficiency of online learning algorithms, exploiting a similar improper learning benefit.
Each projection is a genuine convex program with nontrivial computational cost;
bypassing projections improves efficiency but allows decisions to leave the domain \citep{NeurIPS22:efficient,NeurIPS24:universal-1-projection,colt26lightons}.
Of course, such benefits are not a free lunch; we refer readers to \citep{colt26lightons} for relevant discussions.

As described in \cref{alg:main-reduction}, the domain conversion first projects the lifted decision $\y_t$ onto $\XX$ to produce a proper decision $\x_t$;
then it corrects the gradient $\g_t$ to produce a surrogate gradient $\d_t$ to compensate for the extrapolation error, following the construction of \citet{ICML20:Ashok};
finally, it uses the surrogate gradient to construct closed-form surrogate losses $h_t$, such that:
\emph{(i)} the regret of $\y_t$ on $h_t$ upper bounds the regret of $\x_t$ on $f_t$, and \emph{(ii)} $h_t$ preserves the nice properties of $f_t$, including Lipschitzness, convexity, exp-concavity, and strong convexity.
The guarantee below allows regret transfer on every round, with its proof in \cref{sec:first-order-transfer}.

\begin{thm}[instantaneous regret transfer]
  \label{thm:pointwise-transfer}
  Suppose \cref{asm:domain-gradient,asm:curvature} hold for the selected loss class.
  At every round $t\in[T]$, construct $\x_t,\d_t,h_t$ from any lifted decision $\y_t\in\YY$ by \cref{alg:main-reduction}, using the parameters in \cref{tab:loss-parameters}.
  Then, for all three loss classes, every such pair $(\y_t,\x_t) \in \YY \times \XX$, and every comparator $\u_t\in\XX$,
  \begin{equation}
    \label{eq:main-transfer}
    f_t(\x_t) - f_t(\u_t) \le h_t(\y_t) - h_t(\u_t).
  \end{equation}
\end{thm}

\paragraph{Effects on problem parameters.}
Domain conversion slightly worsens problem parameters, leading to additional constant factors independent of $d$, $T$, and $P_T$.
\Cref{alg:main-reduction} initializes an OIR algorithm on $\YY$ with parameters in the right column of \cref{tab:loss-parameters}.
To accommodate the improper interval comparators, the OIR algorithm runs on a ball with twice the radius of the ball enclosing $\XX$:
\begin{equation*}
  \YY = \BB(D_\YY/2) = \BB(D_\XX), \qquad \XX \subseteq \BB(D_\XX/2), \qquad D_\YY = 2 D_\XX.
\end{equation*}
For exp-concave losses, the left column distinguishes the exp-concavity coefficient $\alpha$ of $f_t$ from the coefficient $\gamma$ used in its quadratic lower bound in \cref{lem:exp-first-order}, which is a restatement of the same argument by \citet{journals/ml/HazanAK07}.
The wrapper sets
\begin{equation}
  \label{eq:main-gamma}
  \gamma
  = \frac{1}{2} \min\lbr{ \alpha, \frac{1}{D_\XX G}, \frac{4}{(D_\XX + D_\YY) G} }
  = \frac{1}{2} \min\lbr{ \alpha, \frac{1}{D_\XX G} }.
\end{equation}
\Cref{tab:loss-parameters} gives the resulting surrogate exp-concavity coefficient $\alpha_h = \gamma/(1 + \gamma D_\YY G)^2$ and quadratic lower-bound coefficient $\gamma_h = \alpha_h/2$.
For strongly convex losses, the conversion preserves the strong-convexity coefficient $\lambda_h = \lambda$.

\paragraph{Hindsight smoothness.}
No smoothness is assumed for the original losses.
In the curved regimes, the closed-form quadratic surrogates $h_t$, specified in \cref{eq:surrogate-exp-concave,eq:surrogate-strongly-convex}, are smooth by construction; we use this property only in the hindsight analysis developed in \cref{sec:main-hindsight}.

\subsubsection{Out-of-domain decision offsets approximation error}
\label{sec:main-hindsight}

Domain conversion lets us control the approximation error in \cref{eq:main-i2d-gap}, now evaluated on surrogate losses.
Focus on the simpler linear surrogate losses $h_t(\y) = \d_t^\trs \y$ and let the interval comparator be $\z_I = \u_n$ for now.
Reindex the interval as $I = [n]$.
Abel summation rewrites the approximation error with cumulative gradients and comparator differences:
\begin{equation*}
  \begin{aligned}
    \gap_I(\z_I) = \sum_{t=1}^n \d_t^\trs (\u_n-\u_t) = \sum_{t=1}^{n-1} \D_t^\trs (\u_{t+1}-\u_t) \le M P_I, \quad
    \D_t = \sum_{s=1}^t \d_s, \quad
    M = \max_{1 \le t < n} \norme{\D_t}.
  \end{aligned}
\end{equation*}
We seek an interval comparator whose loss supplies a compensating negative term of size $M P_I$.
Choose a split $m$ attaining $M$, and take
\begin{equation*}
  \norme{\D_m} = M, \qquad I_1 = [1, m], \qquad I_2 = [m+1, n], \qquad
  \z_1 = \u_n - \frac{P_I}{M} \D_m, \qquad \z_2 = \u_n.
\end{equation*}
The first comparator's contribution relative to $\u_n$ is exactly $-M P_I$, yielding nonpositive approximation error on this interval:
\begin{equation*}
  \gap_I(\z_1, \z_2) = \sum_{t=1}^n \d_t^\trs (\u_n - \u_t) - M P_I \le 0.
\end{equation*}
The quadratic surrogate losses can be handled similarly with smoothness, as the following theorem shows.
Proofs are in \cref{sec:first-order-i2d}.

\begin{thm}[improper interval comparators]
  \label{thm:improper-gap}
  Let $D_\XX > 0$, and set $\YY = \BB(D_\XX)$.
  For $t \in [n]$, let $\u_t \in \BB(D_\XX / 2)$, and write
  $P_{[n]} = \sum_{t=2}^n \norme{\u_t - \u_{t-1}}$.
  In each case below, there is a partition of $[n]$ into at most two consecutive pieces $I_s$ and an interval comparator $\z_s \in \YY$ on each piece, such that:
  \begin{enumerate}[label=(\roman*)]
    \item If every $h_t(\y) = \d_t^\trs \y$ is linear and $P_{[n]} \le D_\XX / 2$, then $\gap_{[n]} \le 0$.
    \item If every $h_t : \YY \to \reals$ is differentiable and $H_h$-smooth, and $\c_t = \u_t - H_h^{-1} \gr h_t(\u_t) \in \YY$, then $\gap_{[n]} \le H_h P_{[n]}^2 n / 8$.
  \end{enumerate}
\end{thm}

\subsubsection{Baby--Wang hindsight partition aggregates ODR}
\label{sec:main-aggregation}

Since \cref{sec:main-hindsight} controls the approximation error in \cref{eq:main-i2d-gap}, it remains to bound the interval complexity in \cref{eq:main-cover-costs}.
Following the hindsight partition principle of \citet{COLT21:baby-exp-concave,AISTATS22:sc-proper}, a greedy partition closes an interval just before its local path-length exceeds $B / |I|^\veps$.
For $B>0$ and $\veps\ge0$, \cref{lem:first-order-partition} gives
\begin{equation*}
  P_{I_j} \le \frac{B}{|I_j|^\veps}, \qquad
  K \le 1 + \sbr{\frac{P_T}{B}}^{\frac{1}{1+\veps}} (2 T)^{\frac{\veps}{1+\veps}}.
\end{equation*}
Split each $I_j$ into the at most two pieces supplied by \cref{thm:improper-gap}, apply \cref{thm:pointwise-transfer}, and decompose the surrogate dynamic regret as in \cref{eq:main-cover-costs}:
\begin{equation}
  \label{eq:main-regret-decomposition}
  \dreg_T(\lbr{\u_t}_{t=1}^T)
  \le \sum_{j=1}^K \sum_{s=1}^{S_j} \sbr{ \ireg_{I_{j,s}}(\z_{j,s}) + \gap_{I_{j,s}}(\z_{j,s}) }.
\end{equation}
For convex losses, set $B = D_\XX / 2$ and $\veps = 0$.
The total approximation error on each $I_j$ is nonpositive by \cref{thm:improper-gap}, and Cauchy--Schwarz sums the interval bounds to
\begin{equation*}
  K = \OO\sbr{ 1 + P_T }, \qquad
  \dreg_T \le \sum_{j=1}^K \sum_{s=1}^{S_j} \Ot\sbr{ \sqrt{|I_{j,s}|} } = \Ot\sbr{ \sqrt{K T} }.
\end{equation*}
For curved losses, let $R_T\ge1$ be their interval-regret bound $\ireg_{I_{j,s}}(\z_{j,s}) \le R_T$.
Set $B = \sqrt{R_T}$ and $\veps = 1/2$.
Then the total approximation error on each $I_j$ is at most $H_h R_T / 8$ by \cref{thm:improper-gap}, and
\begin{equation*}
  K = \OO\sbr{ 1 + R_T^{-\frac{1}{3}} P_T^{\frac{2}{3}} T^{\frac{1}{3}} }, \qquad
  \dreg_T \le \sbr{2+\frac{H_h}{8}} K R_T = \OO\sbr{ R_T + R_T^{\frac{2}{3}} P_T^{\frac{2}{3}} T^{\frac{1}{3}} }.
\end{equation*}
Substituting $R_T = \Ot(d\log|I|)$ or $\Ot(\log|I|)$ proves the two curved rates.

\section{Conclusion}
\label{sec:main-discussion}

This work revisits the prevailing intuition that interval regret is a stronger performance measure than dynamic regret.
We show that the metric-level implication from OIR to ODR depends on domain geometry:
it holds for linear losses on polytopes and their sufficiently small Euclidean smoothings, with geometry-dependent guarantees, but fails on certain smooth domains.
Nevertheless, we show how the algorithmic ability to achieve OIR, combined with domain conversion, yields ODR.
This reduction explains how local-in-time regret minimization can support globally tracking changing comparators, clarifying the connection between two lines of research that have largely developed separately.
As a byproduct, it yields proper exp-concave learning with optimal dynamic regret beyond box domains with an efficient implementation.

The proposed wrapper also gives optimal dynamic regret on every interval; see \cref{sec:interval-dynamic-proof}.
Adapting the construction to higher-order comparator variation \citep{AISTATS23:second-order-path} and to unknown curvature \citep{NIPS16:MetaGrad,NeurIPS24:universal-1-projection} remains an interesting and promising direction.

\clearpage

\section*{AI-use Statement}
\label{sec:ai-methodology}
The authors developed the core one-dimensional reduction argument from interval regret to dynamic regret in March 2026, and used Gemini 3 Pro to help extend it to higher dimensions.
The authors developed the linear-loss lower-bound results in July 2026 with the assistance of GPT-5.6 Pro. The authors used GPT-6 via Codex to help check and polish the manuscript.
The authors take full responsibility for the content and correctness of the paper.

\printbibliography[heading=bibintoc]

\clearpage

\appendix
\crefalias{section}{appendix}
\crefalias{subsection}{appendix}
\crefalias{subsubsection}{appendix}

\section{Exact-cover duality, geometric constructions, and linear-loss extensions}
\label{sec:cert-dichotomy}

This appendix gives the complete arguments for \cref{sec:main-geometry}: homogeneous aggregation characterizes the metric-level implication from OIR to ODR, geometric constructions refute this implication in general, and polyhedral representations recover it for linear losses under suitable geometric conditions.
All regret guarantees here concern the original loss sequence.
In this appendix, we also use \emph{exact cover} as a synonym for \emph{homogeneous aggregation}, as defined in \cref{eq:main-exact-cover}.

\subsection{Proof of \texorpdfstring{\cref{thm:cover-duality}}{Theorem~\ref{thm:cover-duality}}}
\label{sec:cert-lp-completeness}

\begin{prf}[Proof of \cref{thm:cover-duality}]
  Recall $b_I$ and $c_I$ in \cref{thm:cover-duality}, and $\rho$ in \cref{sec:geometry-key-analyses}.
  Note the monotonicity of the interval cost $c_{[t,s]}$ in the sense that $c_{[t,s]} \ge m_t + c_{[t+1,s]}$ for all $1 \le t < s \le T$.
  A singleton satisfies $c_{[t,t]} \ge m_t$.
  For $s > t$, the inequality $b_{[t,s]} \ge m_t + b_{[t+1,s]}$ and the monotonicity of $\rho$ give
  \begin{equation*}
    \begin{aligned}
      c_{[t,s]}
      = b_{[t,s]} + A_T \rho(s - t + 1)
      \ge m_t + b_{[t+1,s]} + A_T \rho(s - t)
      \ge m_t + c_{[t+1,s]}.
    \end{aligned}
  \end{equation*}

  Define the following linear program (LP) for the learner losses $q_t = f_t(\x_t)$:
  \begin{equation*}
    \begin{aligned}
      U^\star = \max_{\q \in \reals^T} \quad
      &\sum_{t=1}^T q_t - \sum_{t=1}^T f_t(\u_t)\\
      \text{subject to} \quad
      &\sum_{t \in I} q_t \le c_I &&(I = [a,b] \subseteq [T]),\\
      &q_t \ge m_t &&(t \in [T]).
    \end{aligned}
  \end{equation*}
  A decision sequence $\lbr{\x_t}_{t=1}^T \in \XX^T$ enters the interval inequalities and dynamic regret only through the scalar learner losses $q_t = f_t(\x_t)$.
  Each $f_t$ is continuous (by differentiability in \cref{asm:domain-gradient}) on the compact connected set $\XX$, so its image is exactly $[m_t,M_t]$.
  Therefore, every feasible $\q$ is attained by some decision sequence satisfying the OIR guarantees; every decision sequence satisfying the OIR guarantees gives a feasible $\q$.
  The LP also clarifies why the interval cost $c_I$ in \cref{thm:cover-duality} is defined differently for singleton intervals:
  it absorbs the constraints $q_t \le M_t$.

  Introduce multipliers $\lambda_I \ge 0$ for the interval upper bounds and $\nu_t \ge 0$ for the lower range bounds.
  Strong duality gives an attained dual optimum
  \begin{equation*}
    \begin{aligned}
      U^\star = \min_{\blambda \ge \bm0,~\bnu \ge \bm0} \quad
      &\sum_{I = [a,b] \subseteq [T]} \lambda_I c_I
      - \sum_{t=1}^T \nu_t m_t
      - \sum_{t=1}^T f_t(\u_t)\\
      \text{subject to} \quad
      &\sum_{I : t \in I} \lambda_I = 1 + \nu_t
      \qquad(t \in [T]).
    \end{aligned}
  \end{equation*}
  Write $w_t = \sum_{I : t \in I} \lambda_I = 1 + \nu_t$ for the total interval weight covering round $t$.

  It remains to show that the dual optimum can be attained with $\bnu = \bm0$, which is equivalent to the right-hand side of \cref{eq:main-cover-duality}.
  Let $t$ be the first round with $w_t > 1$.
  Choose one such interval $[t,s]$ and set $\eps_t = \min\{\lambda_{[t,s]}, w_t-1\}$.
  If $s = t$, remove weight $\eps_t$ from this singleton interval.
  If $s > t$, transfer that weight to $[t+1,s]$, i.e.,
  \begin{equation*}
    \lambda'_{[t,s]} = \lambda_{[t,s]} - \eps_t, \qquad
    \lambda'_{[t+1,s]} = \lambda_{[t+1,s]} + \eps_t.
  \end{equation*}
  In both cases decrease $\nu_t$ by $\eps_t$, preserving dual feasibility.
  The objective changes by
  \begin{equation*}
    \eps_t (m_t - c_{[t,t]}) \le 0 \quad (s=t), \qquad
    \eps_t \sbr{c_{[t+1,s]} - c_{[t,s]} + m_t} \le 0 \quad (s>t).
  \end{equation*}
  Finitely many such operations make $w_t = 1$.
  Continue processing rounds until $\bnu = \bm0$.
\end{prf}

\subsection{Proof of \texorpdfstring{\cref{thm:geometry-separation}}{Theorem~\ref{thm:geometry-separation}}}
\label{sec:cert-negative-constructions}

\begin{prf}[Proof of \cref{thm:geometry-separation}]
  We use the path-length budget $\tau$ from \cref{thm:geometry-separation} and first introduce the common structure of the constructions.
  \begin{itemize}
    \item
      Two boundary contacts of the unit $\ell_p$ ball.
      For some tiny $\delta\in[\tau/(2T), \tau/2]$, the constructions alternate between two boundary contacts $\u_+=[\delta,\phi]^\trs$ and $\u_-=[-\delta,\phi]^\trs$, where $\delta^p+\phi^p=1$.
      Thus each switch has Euclidean length $2\delta$, and the path length satisfies $P_T=2\delta(B-1)\le \tau$.
    \item
      Loss functions.
      Both constructions below have $f_s(\u_s)=0$ and $f_s(\z)\ge0$ for all $\z\in\XX$.
      In particular, $f_+$ uses the unit normal vector $\n_+ = [\delta^{p-1}, \phi^{p-1}]^\trs / \sqrt{\delta^{2p-2} + \phi^{2p-2}}$ at $\u_+$.
    \item
      Block repetition.
      Partition $[T]$ into $B$ consecutive blocks of length $L$, with $B=\lfloor \tau/(2\delta)\rfloor$ and $L=\lceil 2T\delta/\tau\rceil$.
      Without loss of generality, assume $T=BL$.
      Each block repeats $\u_s$ and $f_s$ for $L$ rounds, and the next block switches to $\u_{-s}$ and $f_{-s}$, where $s\in\{+1,-1\}$.
  \end{itemize}
  Fix a homogeneous aggregation with weights $\lbr{\lambda_I}_{I=[a,b]\subseteq [T]}\in\Lambda_T$, and use the interval complexity $R_T$ and approximation error $\gap_T$ from \cref{eq:main-cover-costs}.
  Split its total weighted length into
  \begin{equation*}
    S_{\mathrm{long}} = \sum_{|I|\ge4L}\lambda_I|I|, \qquad
    S_{\mathrm{short}} = \sum_{|I|<4L}\lambda_I|I| = T-S_{\mathrm{long}}.
  \end{equation*}
  For any long interval $I$ of length $|I|\ge4L$, each sign occurs on at least $|I|/4$ rounds.
  Since $f_s(\u_s)=0$ and $f_s(\z)\ge0$, this gives
  \begin{equation}
    \label{eq:cert-long-interval-gap}
    \begin{gathered}
      \gap_I(\z_I)\ge\frac{|I|}{4}\max_{s\in\{+1,-1\}}f_s(\z_I) \quad (|I|\ge4L), \qquad
      \gap_T\ge\sum_{|I|\ge4L}\frac{\lambda_I|I|}{4}\max_{s\in\{+1,-1\}}f_s(\z_I).
    \end{gathered}
  \end{equation}
  Thus long intervals force approximation error, while short intervals force interval complexity.

  \paragraph{Part (i): convex losses.}
  Take the unit $\ell_2$ ball as $\XX$,
  \begin{equation*}
    f_s(\x)=\g_s^\trs(\x-\u_s), \qquad \text{where } \g_s=-\u_s.
  \end{equation*}
  With $A_T=1$ and $\rho(n)=\sqrt n$, the singleton correction in \cref{thm:cover-duality} vanishes.

  For any $\z=[z_1,z_2]^\trs\in\XX$, choose $s\in\{+1,-1\}$ so that $s z_1\le0$.
  Since $z_2\le1$,
  \begin{equation*}
    f_s(\z)=1-s\delta z_1-\phi z_2
    \ge1-\phi
    =\frac{\delta^2}{1+\sqrt{1-\delta^2}}
    =\Theta(\delta^2).
  \end{equation*}

  If $S_{\mathrm{long}}\ge T/2$, \cref{eq:cert-long-interval-gap} gives $\gap_T=\Omega(T\delta^2)$.
  Otherwise, $S_{\mathrm{short}}>T/2$, and
  \begin{equation*}
    R_T
    =\sum_{I}\lambda_I\sqrt{|I|}
    \ge\sum_{|I|<4L}\frac{\lambda_I|I|}{\sqrt{|I|}}
    >\frac{\sum_{|I|<4L}\lambda_I|I|}{\sqrt{4L}}
    >\frac{T}{2\sqrt{4L}}
    =\Omega\sbr{\sqrt{\frac{T\tau}{\delta}}}.
  \end{equation*}
  All approximation error terms are nonnegative, so every exact cover satisfies
  \begin{equation*}
    R_T+\gap_T
    =\Omega\sbr{\min\lbr{T\delta^2,\sqrt{T\tau/\delta}}}
    \stackrel{\delta=(\tau/T)^{\frac{1}{5}}}{=}\Omega(\tau^{\frac{2}{5}}T^{\frac{3}{5}}).
  \end{equation*}

  \paragraph{Part (ii): exp-concave and strongly convex losses.}

  Take the unit $\ell_{3/2}$ ball as $\XX$,
  \begin{equation*}
    f_s(\x)=\g_s^\trs(\x-\u_s)+\frac12\norme{\x-\u_s}^2, \qquad
    \g_s=-\frac{(s\sqrt\delta,\sqrt\phi)^\trs}{\sqrt{\delta+\phi}}.
  \end{equation*}
  The domain is contained in the Euclidean unit ball, and hence $\norme{\gr f_s(\x)}\le1+\norme{\x-\u_s}\le3$.
  Its Hessian is the identity matrix; thus $\gr^2 f_s(\x)\succeq(1/9)\gr f_s(\x)\gr f_s(\x)^\trs$.
  Thus the losses are $1$-strongly convex and $1/9$-exp-concave.
  Moreover, $f_s(-\u_s)\ge2/\sqrt{\delta+\phi}\ge\sqrt2>1$, so the singleton correction again vanishes at coefficient one.

  For any $\z\in\XX$, choose $s\in\{+1,-1\}$ so that $s z_1\le0$.
  Using $z_2\le1$ gives
  \begin{equation*}
    \begin{aligned}
      f_s(\z)
      \ge\frac{1-s\sqrt\delta z_1-\sqrt\phi z_2}{\sqrt{\delta+\phi}}
      \ge\frac{1-\sqrt\phi}{\sqrt{\delta+\phi}}
      =\frac{1-(1-\delta^{\frac{3}{2}})^{\frac{1}{3}}}{\sqrt{\delta+\phi}}
      =\Theta(\delta^{\frac{3}{2}}).
    \end{aligned}
  \end{equation*}
  Here $1\le\delta+\phi\le2$ and $v/3\le1-(1-v)^{1/3}\le v$ for $0\le v\le1$.

  If $S_{\mathrm{long}}\ge T/2$, the approximation error is at least $\Omega(T\delta^{3/2})$.
  Otherwise, $\rho(n)=1$ gives
  \begin{equation*}
    R_T
    =\sum_{I}\lambda_I
    \ge\sum_{|I|<4L}\frac{\lambda_I|I|}{|I|}
    >\frac{\sum_{|I|<4L}\lambda_I|I|}{4L}
    >\frac{T}{2(4L)}
    =\Omega(\tau/\delta).
  \end{equation*}
  Therefore every homogeneous aggregation satisfies
  \begin{equation*}
    R_T+\gap_T
    =\Omega\sbr{\min\lbr{T\delta^{\frac{3}{2}},\tau/\delta}}
    \stackrel{\delta=(\tau/T)^{\frac{2}{5}}}{=}\Omega(\tau^{\frac{3}{5}}T^{\frac{2}{5}}).
  \end{equation*}
\end{prf}

\subsection{Extending Luo--Schapire arguments to polytopes and their smoothings}
\label{sec:cert-luo-schapire}

We revisit the tradeoff between interval complexity and approximation error in \cref{sec:geometry-key-analyses} through the geometry of polytopes.
For linear losses, the Luo--Schapire decomposition yields exact covers with zero approximation error, while their interval complexity is controlled by a geometric coefficient $G_\QQ$ \citep{COLT15:Luo-AdaNormalHedge}.
This coefficient measures how rapidly the vertex weights representing a comparator must change as the comparator moves within $\QQ$.
It depends only on the geometry of $\QQ$ and can be bounded by a constant for standard simplices and boxes in fixed dimension; \cref{lem:cert-barycentric-rho} gives a quantitative bound for general polytopes.
The possible growth of $G_\QQ$ reinforces the geometric insight underlying our negative result.
Although polyhedral representations eliminate approximation error, their interval-complexity bound can deteriorate as the domain approaches a smooth convex body.
For example, for regular $N$-gons inscribed in the unit circle, adjacent vertices have distance $2\sin(\pi/N)$, whereas their representing probability vectors have $\ell_1$ distance $2$, forcing $G_\QQ\ge1/\sin(\pi/N)$.
Thus, even in two dimensions, regular $T$-gons can make the resulting dynamic-regret bound vacuous.
This is consistent with the obstruction on the smooth balls in \cref{thm:geometry-separation}: the polyhedral implication does not extend uniformly as these polygons approach smooth bodies.

Finally, \cref{thm:cert-smoothed-polygon} extends the construction to $\QQ+\eps\BB$, preserving the interval-complexity bound while introducing approximation error at most $\eps T$.
This shows that the geometric separation is not a strict distinction between polytopes and non-polytopes, but depends on whether the domain admits a polyhedral representation accurate enough at the resolution required by the target regret bound.

We write an exact cover as a family $\lbr{(\lambda_k,I_k,\z_k)}_{k=1}^K$ of positive-weight interval--comparator triples.
Write $\Delta_N$ for the probability simplex on $N$ points:
\begin{equation*}
  \Delta_N \triangleq \lbr{ \p \in \reals_+^N : \sum_{j=1}^N p_j=1 }.
\end{equation*}
For a polytope $\QQ \subseteq \reals^d$ with vertices $\lbr{\w_j}_{j=1}^N$, a barycentric lift $\pi : \QQ \to \Delta_N$ represents each point by probability weights on the vertices. It is $G_\QQ$-Lipschitz if
\begin{equation}
  \label{eq:cert-barycentric-lift}
  \u = \sum_{j=1}^N \pi_j(\u) \w_j, \qquad
  \norm{\pi(\u)-\pi(\v)}_1 \le G_\QQ \norme{\u-\v} \quad \text{for all }(\u,\v)\in\QQ^2.
\end{equation}

\begin{thm}[polyhedral covers with zero approximation error]
  \label{thm:cert-polyhedral-reduction}
  Let $\QQ \subseteq \reals^d$ be a polytope.
  It admits a barycentric lift $\pi$ satisfying \cref{eq:cert-barycentric-lift} with a finite coefficient $G_\QQ$.
  For any linear losses $f_t(\x)=\g_t^\trs\x$ and path $\lbr{\u_t}_{t=1}^T \in \QQ^T$ of path-length $P_T$, \cref{alg:cert-birth-death} with this lift returns an exact cover whose $R_T$ and $\gap_T$ satisfy
  \begin{equation}
    \label{eq:cert-polyhedral-cover}
    R_T \le \sqrt{T \sbr{1 + \frac{G_\QQ}{2} P_T}}, \qquad
    \gap_T = 0.
  \end{equation}
  Consequently, any decision sequence satisfying \cref{eq:oir-guarantees} with $A_T \ge 1$ and $\rho(n)=\sqrt n$ obeys
  \begin{equation}
    \label{eq:cert-polyhedral-odr}
    \dreg_T(\lbr{\u_t}_{t=1}^T)
    \le A_T \sqrt{T \sbr{1 + \frac{G_\QQ}{2} P_T}}.
  \end{equation}
\end{thm}

\begin{prf}[Proof of \cref{thm:cert-polyhedral-reduction}]
  Use the barycentric lift from \cref{lem:cert-polyhedral-lift} in \cref{alg:cert-birth-death}.
  By \cref{lem:cert-birth-death}, the output is an exact cover.
  By \cref{eq:cert-barycentric-lift,eq:cert-birth-mass},
  \begin{equation*}
    W_T
    \triangleq \sum_k \lambda_k
    = 1 + \frac12 \sum_{t=2}^T \norm{\pi(\u_t) - \pi(\u_{t-1})}_1
    \le 1 + \frac{G_\QQ}{2} P_T, \qquad
    R_T
    \le \sqrt{T \sbr{1 + \frac{G_\QQ}{2} P_T}}.
  \end{equation*}
  For every linear loss sequence, \cref{eq:cert-barycenter-reconstruction} gives
  \begin{equation*}
    \gap_T
    = \sum_{t=1}^T \g_t^\trs \sbr{\sum_{k : t \in I_k} \lambda_k \z_k - \u_t}
    = 0.
  \end{equation*}
  This proves \cref{eq:cert-polyhedral-cover}; combining it with \cref{eq:main-cover-costs} proves \cref{eq:cert-polyhedral-odr}.
\end{prf}

\begin{algorithm}[th]
  \caption{Luo--Schapire vertex decomposition}
  \label{alg:cert-birth-death}
  \begin{algorithmic}
    \REQUIRE A polytope $\QQ$ with vertices $\lbr{\w_j}_{j=1}^N$, comparators $\lbr{\u_t}_{t=1}^T \in \QQ^T$, and a lift $\pi$ as in \cref{eq:cert-barycentric-lift}.
    \ENSURE A family of weighted interval--vertex triples $\lbr{(\lambda_k, I_k, \z_k)}_{k=1}^K$.
    \STATE Set $\p_t = \pi(\u_t)$ for $t \in [T]$ and initialize $a_{1,j} = p_{1,j}$ for $j \in [N]$.
    \STATE Here $a_{s,j}$ tracks the surviving mass born at round $s$ at vertex $\w_j$.
    \FOR{$j=1$ to $N$}
    \FOR{$t=2$ to $T$}
    \STATE Set $c_{t,j} = \min\{p_{t-1,j}, p_{t,j}\}$ and
    $\theta_{t,j} = c_{t,j} / p_{t-1,j}$ if $p_{t-1,j} > 0$, and $\theta_{t,j} = 0$ otherwise.
    \FOR{every $s < t$ with active mass $a_{s,j} > 0$}
    \STATE Emit $((1 - \theta_{t,j}) a_{s,j}, [s, t-1], \w_j)$ if its weight is positive, and replace $a_{s,j}$ by $\theta_{t,j} a_{s,j}$.
    \ENDFOR
    \STATE Start the new active mass $a_{t,j} = p_{t,j} - c_{t,j}$.
    \ENDFOR
    \ENDFOR
    \STATE Emit every positive surviving mass as $(a_{s,j}, [s,T], \w_j)$.
  \end{algorithmic}
\end{algorithm}

\begin{figure}[th]
  \centering
  \begin{tikzpicture}[
      x=0.9cm, y=0.9cm,
      component/.style={fill=white, inner sep=1.5pt, font=\normalsize}
    ]
    \foreach \offset/\masscolor in {0/orange!85!black,3/red!75!black,6/violet!90!black} {
      \path[fill=\masscolor, fill opacity=0.18]
      (\offset,0)
      plot[domain=0:11, samples=101, variable=\masscoord]
      ({\masscoord+\offset},
      {4.8*(exp(-((\masscoord-5.5)/3.2)^2)-exp(-(5.5/3.2)^2))/(1-exp(-(5.5/3.2)^2))})
      -- cycle;
    }
    \foreach \offset/\masscolor in {0/orange!85!black,3/red!75!black,6/violet!90!black} {
      \draw[draw=\masscolor, line width=1pt]
      plot[domain=0:11, samples=101, variable=\masscoord]
      ({\masscoord+\offset},
      {4.8*(exp(-((\masscoord-5.5)/3.2)^2)-exp(-(5.5/3.2)^2))/(1-exp(-(5.5/3.2)^2))});
    }
    \draw[black!70, line width=0.6pt] (0,0) -- (17,0);

    \node[text=orange!85!black] at (5.5,5.25) {$\p_1=\pi(\u_1)$};
    \node[text=red!75!black] at (8.5,5.25) {$\p_2=\pi(\u_2)$};
    \node[text=violet!90!black] at (11.5,5.25) {$\p_3=\pi(\u_3)$};

    \node[component] at (4.85,3.0) {$(\lambda_1,[1,1],\z_1)$};
    \node[component] at (8.5,3.83) {$(\lambda_2,[2,2],\z_2)$};
    \node[component] at (12.15,3.0) {$(\lambda_3,[3,3],\z_3)$};
    \node[component] at (6.8,1.3) {$(\lambda_4,[1,2],\z_4)$};
    \node[component] at (10.2,1.3) {$(\lambda_5,[2,3],\z_5)$};
    \node[component] at (8.5,0.35) {$(\lambda_6,[1,3],\z_6)$};
  \end{tikzpicture}

  \caption{
    A schematic Luo--Schapire decomposition for three rounds.
    The horizontal axis represents distribution support and the vertical axis represents density; the curves are a continuous illustration of the discrete vertex weights.
    Overlap illustrates mass that can be retained across rounds, without specifying the proportional splitting used by \cref{alg:cert-birth-death}.
    The labels show possible lifetimes of mass pieces and their interval-comparator triples $(\lambda_k,I_k,\z_k)$.
    For example, $(\lambda_4,[1,2],\z_4)$ represents mass added at round $1$, retained at round $2$, and removed and output at the transition to round $3$.
  }
  \label{fig:cert-luo-schapire}
\end{figure}
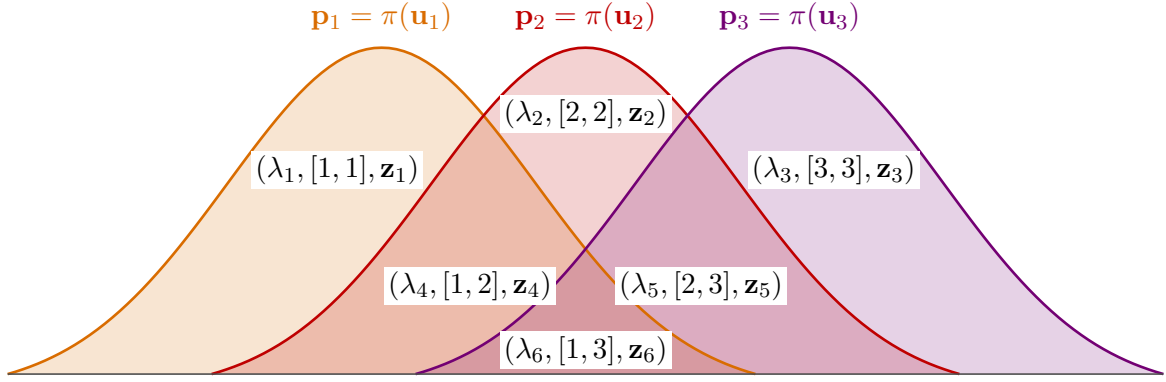

\Cref{alg:cert-birth-death} is an equivalent formulation of the OIR-to-ODR analysis of \citet{COLT15:Luo-AdaNormalHedge}.
\Cref{fig:cert-luo-schapire} illustrates the idea behind \cref{alg:cert-birth-death}.
At each vertex, an increase in mass starts a new group of mass pieces; a decrease removes the same fraction from every surviving group.
Each piece contributes its weight to the cover during the rounds it survives, until removal or the horizon.

\begin{lem}[Luo--Schapire cover]
  \label{lem:cert-birth-death}
  Let $\QQ \subseteq \reals^d$ be a polytope with vertices $\lbr{\w_j}_{j=1}^N$, and let $\pi:\QQ\to\Delta_N$ satisfy \cref{eq:cert-barycentric-lift}.
  For any path $\lbr{\u_t}_{t=1}^T\in\QQ^T$, set $\p_t=\pi(\u_t)$.
  The output $\lbr{(\lambda_k,I_k,\z_k)}_{k=1}^K$ of \cref{alg:cert-birth-death} is an exact cover on $[T]$ and $\QQ$, and its interval complexity $R_T$ from \cref{eq:main-cover-costs} satisfies
  \begin{equation}
    \label{eq:cert-birth-mass}
    W_T \triangleq \sum_k \lambda_k = 1 + \frac12 \sum_{t=2}^T \norm{\p_t - \p_{t-1}}_1, \qquad
    R_T^2 \le T W_T.
  \end{equation}
  At every $t\in[T]$,
  \begin{equation}
    \label{eq:cert-barycenter-reconstruction}
    \sum_{k : t \in I_k} \lambda_k \z_k = \u_t.
  \end{equation}
\end{lem}

\begin{prf}[Proof of \cref{lem:cert-birth-death}]
  At time $t$, at vertex $j$, the update retains mass $c_{t,j}$ and adds $p_{t,j}-c_{t,j}$, recovering total mass $p_{t,j}$.
  Pieces whose lifetimes contain $t$ partition exactly the mass present at round $t$, so
  \begin{equation*}
    \sum_{k : t \in I_k,~ \z_k = \w_j} \lambda_k = p_{t,j}.
  \end{equation*}
  Summing over $j$ proves the exact-cover constraint, and further multiplying by $\w_j$ proves \cref{eq:cert-barycenter-reconstruction}.

  Each mass piece is output exactly once, at death or at the horizon, so $W_T$ equals the total mass born.
  The initial mass is one, and each round $t\ge2$ adds
  \begin{equation*}
    \sum_{j=1}^N \max\{p_{t,j} - p_{t-1,j}, 0\}
    = \frac12 \norm{\p_t - \p_{t-1}}_1,
  \end{equation*}
  because both probability vectors have unit mass.
  This proves the identity for $W_T$.
  Finally, \cref{eq:main-exact-cover} and Cauchy--Schwarz give
  \begin{equation*}
    R_T^2
    \le \sbr{\sum_k \lambda_k} \sbr{\sum_k \lambda_k |I_k|}
    = W_T T.
  \end{equation*}
  \Cref{fig:cert-luo-schapire} gives a geometric interpretation of these two factors.
  The first factor counts each mass piece once, regardless of how many rounds it survives, so it equals the initial unit mass plus all newly introduced mass.
  The second factor counts each mass piece once per round it survives and equals $T$, since every round has total mass one.
\end{prf}

\begin{lem}[polyhedral barycentric lift]
  \label{lem:cert-polyhedral-lift}
  Let $\QQ \subseteq \reals^d$ be a polytope with vertices $\lbr{\w_j}_{j=1}^N$.
  There is a map $\pi:\QQ\to\Delta_N$ satisfying \cref{eq:cert-barycentric-lift} with a finite Lipschitz coefficient $G_\QQ$.
  For full-dimensional polytopes, \cref{lem:cert-barycentric-rho} quantifies $G_\QQ$.
\end{lem}

\begin{prf}[Proof of \cref{lem:cert-polyhedral-lift}]
  Triangulate $\QQ$ in its affine hull using its vertices $\lbr{\w_j}_{j=1}^N$.
  The barycentric coordinates on each simplex, extended by zero to the other vertices, define a continuous map $\pi : \QQ \to \Delta_N$ and satisfy $\u = \sum_j \pi_j(\u) \w_j$.
  On each simplex, $\pi$ is affine.
  There are finitely many simplices, so the largest $\ell_2$-to-$\ell_1$ operator norm of their linear parts is finite; take this as $G_\QQ$.
\end{prf}

\begin{lem}[quantitative barycentric constant]
  \label{lem:cert-barycentric-rho}
  Let $\QQ\subseteq\reals^d$ be a $d$-dimensional polytope, with $D_\QQ=\diam(\QQ)$.
  Call $F \subseteq \QQ$ an exposed face if $F = \argmax_{\x \in \QQ} \n^\trs \x$ for some nonzero $\n$.
  Define
  \begin{equation}
    \label{eq:cert-face-separation}
    \rho_\QQ \triangleq \inf_{\substack{
        m \ge 2,~ \lbr{F_i}_{i=1}^m \text{ distinct exposed faces of }\QQ,~ \u_i \in F_i,~ \bigcap_{i=1}^m F_i = \varnothing
    }} \diam\lbr{\u_i}_{i=1}^m,
  \end{equation}
  with $\inf\varnothing = +\infty$.
  Then $\rho_\QQ > 0$, and $\QQ$ admits a barycentric lift satisfying \cref{eq:cert-barycentric-lift} with coefficient
  \begin{equation}
    \label{eq:cert-barycentric-rho}
    G_\QQ \le \frac{(d+1) 2^{d-1} D_\QQ^{d-1}}{\rho_\QQ^d}.
  \end{equation}
  Consequently, for any positive integer $T$, linear losses $f_t(\x)=\g_t^\trs\x$, and path $\lbr{\u_t}_{t=1}^T\in\QQ^T$ with path-length $P_T$, there is an exact cover whose $R_T$ and $\gap_T$ from \cref{eq:main-cover-costs} satisfy
  \begin{equation}
    \label{eq:cert-barycentric-rho-cover}
    R_T \le \sqrt{T\sbr{ 1 + \frac{(d+1) 2^{d-2} D_\QQ^{d-1}}{\rho_\QQ^d} P_T }}, \qquad
    \gap_T = 0.
  \end{equation}
\end{lem}

\begin{prf}[Proof of \cref{lem:cert-barycentric-rho}]
  Fix a collection of distinct faces with empty intersection and choose one point from each face.
  These points cannot have arbitrarily small diameter: otherwise, compactness would give a limiting point common to all the faces, contradicting the empty intersection.
  Thus, for each collection of faces, the diameter of the chosen points has a positive lower bound.
  A polytope has finitely many exposed faces.
  Taking the minimum over the finitely many such collections in \cref{eq:cert-face-separation} gives $\rho_\QQ > 0$.
  Put $\rho = \rho_\QQ$.

  For each positive-dimensional face $F$ of $\QQ$, and for $F = \QQ$ itself, let $\c_F$ and $r_F$ be the center and radius of a largest Euclidean ball contained in $F$.
  Then \cref{eq:cert-face-separation} gives $\rho \le 2r_F$ and every edge has length at least $\rho$.

  Triangulate $\QQ$ by coning from $\c_\QQ$.
  Each such simplex $\Delta$ corresponds to a nested sequence of faces $F_1\subset\cdots\subset F_d=\QQ$, where $\dim F_j=j$.
  The $j$-dimensional volume $\operatorname{vol}_j(F_j) \ge \frac{1}{j} r_{F_j} \operatorname{vol}_{j-1}(F_{j-1})$;
  $\operatorname{vol}_1(F_1)$ is the length of an edge of $\QQ$, which is at least $\rho$.
  Using $r_{F_j} \ge \rho/2$ for every $j$,
  \begin{equation*}
    \operatorname{vol}_d(\Delta)
    \ge \rho \prod_{j=2}^{d} \frac{r_{F_j}}{j}
    \ge \frac{\rho^d}{2^{d-1} d!}.
  \end{equation*}
  Every opposite $(d-1)$-face of $\Delta$ has diameter at most $D_\QQ$, so Hadamard's inequality bounds its volume by $D_\QQ^{d-1} / (d-1)!$.
  Therefore every altitude of $\Delta$ is at least $\rho^d / (2^{d-1} D_\QQ^{d-1})$.

  For every node $\q$ of this triangulation, fix a probability vector $\balpha^{\q} \in \Delta_N$ whose barycenter is $\q$, taking $\balpha^{\w_j} = \e_j$ at each original vertex.
  On a simplex $\Delta$ with nodes $\lbr{\q_\ell}_{\ell = 0}^d$ and scalar barycentric coordinates $\lbr{\lambda_\ell}_{\ell = 0}^d$, set
  \begin{equation*}
    \pi(\x) \triangleq \sum_{\ell = 0}^d \lambda_\ell(\x) \balpha^{\q_\ell}.
  \end{equation*}
  The nodal values agree across adjacent simplices, so $\pi : \QQ \to \Delta_N$ is continuous, and
  \begin{equation*}
    \sum_{j=1}^N \pi_j(\x) \w_j
    = \sum_{\ell = 0}^d \lambda_\ell(\x) \q_\ell
    = \x.
  \end{equation*}
  Since $\lambda_\ell$ is affine, equals $1$ at $\q_\ell$, and vanishes on the opposite face, we have
  \begin{equation*}
    \norme{\gr\lambda_\ell}
    = \frac{1}{\operatorname{dist}(\q_\ell,\operatorname{aff}\{\q_j:j\ne\ell\})}
    \le \frac{2^{d-1} D_\QQ^{d-1}}{\rho^d}.
  \end{equation*}
  Since $\norm{\balpha^{\q_\ell}}_1 = 1$, the linear part of $\pi$ on $\Delta$ satisfies, for every $\v \in \reals^d$,
  \begin{equation*}
    \norm{\sum_{\ell = 0}^d (\gr\lambda_\ell)^\trs \v \, \balpha^{\q_\ell}}_1
    \le \frac{(d+1) 2^{d-1} D_\QQ^{d-1}}{\rho^d} \norme{\v}.
  \end{equation*}
  Subdividing the line segment between any two points of $\QQ$ at simplex boundaries and summing the bounds over the resulting subsegments proves \cref{eq:cert-barycentric-rho} globally.
\end{prf}

\begin{thm}[smoothed polytopes]
  \label{thm:cert-smoothed-polygon}
  Let $\QQ\subseteq\reals^d$ be a polytope, let $\BB$ be the Euclidean unit ball, and let $\eps\ge0$.
  Define the $\eps$-smoothed polytope
  \begin{equation*}
    \XX_\eps=\QQ+\eps\BB
    =\{\q+\eps\b:\q\in\QQ,\ \b\in\BB\}.
  \end{equation*}
  For every positive integer $T$, linear losses $f_t(\x)=\g_t^\trs\x$ with $\norme{\g_t}\le1$, and comparator path $\lbr{\u_t}_{t=1}^T\in\XX_\eps^T$, there is an exact cover on $[T]$ and $\XX_\eps$ satisfying
  \begin{equation}
    \label{eq:cert-smoothed-cover}
    R_T\le\sqrt{T\sbr{1+\frac{G_\QQ}{2}P_T}},
    \qquad \gap_T\le\eps T,
  \end{equation}
  where $G_\QQ$ is a barycentric Lipschitz coefficient from \cref{lem:cert-polyhedral-lift}.
  Consequently, every decision sequence on $\XX_\eps$ satisfying \cref{eq:oir-guarantees} with $\rho(n)=\sqrt n$ obeys
  \begin{equation*}
    \dreg_T(\lbr{\u_t}_{t=1}^T)
    \le A_T\sqrt{T\sbr{1+\frac{G_\QQ}{2}P_T}}+\eps T.
  \end{equation*}
  When $A_T$ is polylogarithmic in $T$ and the smoothing radius satisfies $\eps=\OO(\sqrt{(1+G_\QQ P_T)/T})$, the bound has order $\Ot(\sqrt{T(1+G_\QQ P_T)})$.
\end{thm}

\begin{prf}[Proof of \cref{thm:cert-smoothed-polygon}]
  Project each comparator onto $\QQ$: set $\q_t=\Pi_\QQ[\u_t]$.
  By the definition of $\XX_\eps$, $\norme{\q_t-\u_t}\le\eps$.
  Euclidean projection is nonexpansive, so
  \begin{equation*}
    \sum_{t=2}^T\norme{\q_t-\q_{t-1}}
    \le\sum_{t=2}^T\norme{\u_t-\u_{t-1}}=P_T.
  \end{equation*}
  Apply \cref{alg:cert-birth-death} to the projected path with a barycentric lift of $\QQ$.
  Its interval comparators lie in $\QQ\subseteq\XX_\eps$, and its barycenter at time $t$ is $\q_t$.
  The interval complexity therefore satisfies the first inequality in \cref{eq:cert-smoothed-cover}, while linearity gives
  \begin{equation*}
    \gap_T
    =\sum_{t=1}^T\g_t^\trs(\q_t-\u_t)
    \le\sum_{t=1}^T\norme{\q_t-\u_t}
    \le\eps T.
  \end{equation*}
  The regret consequence follows from \cref{eq:main-cover-costs}.
\end{prf}

\section{OIR interface and complete proofs}
\label{sec:algorithm-details}

This appendix proves the algorithmic reduction in \cref{sec:main-reduction}.
The OIR guarantees apply to the surrogate losses on the enlarged domain, while instantaneous regret transfer yields ODR for the original losses with decisions in $\XX$.
The proofs combine domain conversion, control of approximation error by out-of-domain interval comparators, and the Baby--Wang hindsight partition, following \cref{sec:reduction-key-analyses}.

\subsection{Proof of \texorpdfstring{\cref{thm:algorithm-reduction}}{Theorem~\ref{thm:algorithm-reduction}}}
\label{sec:first-order-guarantee}

\begin{prf}[Proof of \cref{thm:algorithm-reduction}]
  Fix a dynamic comparator sequence $\lbr{\u_t}_{t=1}^T\in\XX^T$.
  By \cref{lem:surrogate-regularity}, the surrogate losses meet the loss assumptions of the OIR interface formalized in \cref{dfn:first-order-algorithm}.
  Below, $I_j$ denotes a hindsight partition interval, $I_{j,s}$ its $s$-th piece, and $\z_{j,s}\in\YY$ the interval comparator on that piece.
  For every $t \in I_{j,s}$,
  \cref{eq:main-transfer} and addition and subtraction of $h_t(\z_{j,s})$ give
  \begin{equation*}
    f_t(\x_t) - f_t(\u_t)
    \le h_t(\y_t) - h_t(\z_{j,s}) + h_t(\z_{j,s}) - h_t(\u_t).
  \end{equation*}
  Summing over rounds within each piece, then over pieces and hindsight intervals, gives \cref{eq:main-regret-decomposition}.
  Following \cref{sec:main-aggregation} finishes the proof.
\end{prf}

\paragraph{OIR algorithm interface.}
The input algorithm predicts on $\YY=\BB(D_\XX)$ and receives the full surrogate loss from \cref{eq:surrogate-convex,eq:surrogate-exp-concave,eq:surrogate-strongly-convex} after its prediction, as specified in \cref{alg:main-reduction}.
Existing interval algorithms provide this interface for convex losses \citep{ICML15:Daniely-adaptive,AISTATS17:coin-betting-adaptive} and all three classes \citep{ICML18:zhang-dynamic-adaptive,NIPS21:dual-adaptive}.

\begin{dfn}[OIR algorithm interface for interval comparators]
  \label{dfn:first-order-algorithm}
  Fix a constant $D_\YY>0$ and a compact convex domain $\YY \subseteq \reals^d$ of diameter at most $D_\YY$.
  At each round $t\in[T]$, the algorithm $\AA$ returns $\y_t \in \YY$ and then receives the full differentiable convex loss $h_t : \YY \to \reals$.
  For an interval $I\subseteq[T]$ and an interval comparator $\z_I \in \YY$, define
  \begin{equation*}
    \ireg_I(\z_I) \triangleq \sum_{t \in I} h_t(\y_t) - \sum_{t \in I} h_t(\z_I).
  \end{equation*}
  Using the surrogate-loss notation of \cref{tab:loss-parameters}, fix a Lipschitz coefficient $G_h > 0$ and, in the curved modes, an exp-concavity coefficient $\alpha_h > 0$ or a strong-convexity coefficient $\lambda_h > 0$.
  The loss class and these parameters are specified at initialization and apply to every loss $h_t$.
  Simultaneously for every $I$ and $\z_I$, the OIR guarantees are
  \begin{center}
    \begin{tabular}{@{}cc@{}}
      \toprule
      Conditions on $h_t$ & Upper bound on $\ireg_I(\z_I)$ \\
      \midrule
      Convex and $G_h$-Lipschitz & $\Ot(\sqrt{|I|})$ \\
      Additionally $\alpha_h$-exp-concave & $\Ot(d\log|I|)$ \\
      Additionally $\lambda_h$-strongly convex & $\Ot(\log|I|)$ \\
      \bottomrule
    \end{tabular}
  \end{center}
  The hidden constants are independent of $d$, $T$, $I$, and $\z_I$; $\Ot$ hides polylogarithmic factors in $T$.
\end{dfn}

\begin{algorithm}[th]
  \caption{Hindsight partition by local path-length}
  \label{alg:partition}
  \begin{algorithmic}
    \REQUIRE Positive integers $T,d$, a sequence $\lbr{\u_t}_{t=1}^T\in(\reals^d)^T$, scale $B>0$, and exponent $\veps\ge0$.
    \ENSURE A partition of $[T]$ into intervals $\lbr{I_j}_{j=1}^K$.
    \STATE Set the current left endpoint $a = 1$ and initialize an empty list.
    \FOR{$t = 2$ to $T$}
    \IF{$P_{[a,t]} > B/(t-a+1)^\veps$}
    \STATE Append $[a,t-1]$ to the list and set $a = t$.
    \ENDIF
    \ENDFOR
    \STATE Append $[a,T]$ and denote the resulting intervals by $\lbr{I_j}_{j=1}^K$.
  \end{algorithmic}
\end{algorithm}

\begin{lem}[partition for first differences]
  \label{lem:first-order-partition}
  Let $B>0$ and $\veps\ge0$, and let $\lbr{\u_t}_{t=1}^T\in(\reals^d)^T$.
  With $P_I=\sum_{t=a+1}^b\norme{\u_t-\u_{t-1}}$ for $I=[a,b]$, \cref{alg:partition} returns a partition $\lbr{I_j}_{j=1}^K$ of $[T]$ satisfying, for every $j\in[K]$,
  \begin{equation}
    \label{eq:first-order-partition-local}
    P_{I_j} \le B/|I_j|^\veps.
  \end{equation}
  For nontrivial $P_T>0$,
  \begin{equation}
    \label{eq:first-order-partition-count}
    K \le \left\lceil (P_T/B)^{\frac{1}{1+\veps}} (2T)^{\frac{\veps}{1+\veps}} \right\rceil.
  \end{equation}
\end{lem}

\begin{prf}[Proof of \cref{lem:first-order-partition}]
  Starting at $a = 1$, close the current interval immediately before the first $t$ for which $P_{[a, t]} > B/(t - a + 1)^\veps$, then restart at $t$.
  This proves \cref{eq:first-order-partition-local}.

  Each of the first $K-1$ cuts counts the path increments in the completed interval and its outgoing step.
  No increment is counted twice, so
  \begin{equation*}
    P_T
    > B \sum_{j=1}^{K-1} (|I_j| + 1)^{-\veps}
    \ge B \frac{(K-1)^{1+\veps}}{(2 T)^\veps},
  \end{equation*}
  giving \cref{eq:first-order-partition-count}.
  The last step is Jensen's inequality, using $\sum_{j=1}^{K-1}(|I_j|+1)\le T+K-2<2T$.
\end{prf}

\subsection{Proof of \texorpdfstring{\cref{thm:pointwise-transfer}}{Theorem~\ref{thm:pointwise-transfer}}}
\label{sec:first-order-transfer}

\begin{prf}[Proof of \cref{thm:pointwise-transfer}]
  Fix $\y_t\in\YY$ and $\u_t\in\XX$.
  Construct $\x_t$ and $\d_t$ as in \cref{alg:main-reduction}, and use the applicable surrogate loss in \cref{eq:surrogate-convex,eq:surrogate-exp-concave,eq:surrogate-strongly-convex}, with $D_\YY=2D_\XX$.
  By \cref{lem:projection-correction}, $\norme{\d_t}\le G$ and
  $\g_t^\trs(\x_t-\u_t)\le\d_t^\trs(\y_t-\u_t)$.
  Together with convexity, this proves the convex case.

  For exp-concave losses, the definition of $\gamma$ in \cref{eq:main-gamma} satisfies
  $0<\gamma\le\frac12\min\{\alpha,1/(D_\XX G)\}$.
  Apply \cref{lem:exp-first-order} on $\XX$ and set $A = \g_t^\trs (\x_t - \u_t)$ and $B = \d_t^\trs (\y_t - \u_t)$.
  Then $A \le B$ and $B \le (D_\XX + D_\YY) G/2 \le 1/\gamma$.
  Since $a \mapsto a - \gamma a^2/2$ is nondecreasing up to $1/\gamma$,
  \begin{equation*}
    f_t(\x_t) - f_t(\u_t)
    \le A - \frac{\gamma}{2} A^2
    \le B - \frac{\gamma}{2} B^2 = h_t(\y_t) - h_t(\u_t).
  \end{equation*}
  Finally, strong convexity and \cref{lem:projection-correction} give
  \begin{equation*}
    \begin{aligned}
      f_t(\x_t) - f_t(\u_t)
      &\le \g_t^\trs (\x_t-\u_t) - \frac{\lambda}{2} \norme{\x_t-\u_t}^2\\
      &\le \d_t^\trs (\y_t-\u_t) + \frac{\lambda}{2} \norme{\y_t-\x_t}^2 - \frac{\lambda}{2} \norme{\u_t-\x_t}^2
      = h_t(\y_t) - h_t(\u_t).
    \end{aligned}
  \end{equation*}
\end{prf}

\begin{lem}[projection correction {\citep{ICML20:Ashok}}]
  \label{lem:projection-correction}
  Let $\XX\subseteq\reals^d$ be closed and convex, and let $(\y_t,\g_t)\in(\reals^d)^2$ with $\norme{\g_t}\le G$ for some $G\ge0$.
  Define $\x_t$, $\n_t$, and $\d_t$ by \cref{alg:main-reduction}.
  The corrected gradient satisfies, for every $\u\in\XX$,
  \begin{equation*}
    \norme{\d_t} \le \norme{\g_t} \le G, \qquad
    \g_t^\trs (\x_t - \u) \le \d_t^\trs (\y_t - \u).
  \end{equation*}
\end{lem}

\begin{prf}[Proof of \cref{lem:projection-correction}]
  We give the argument of \citet{ICML20:Ashok} for a fixed domain in our notation.
  Fix $\u_t \in \XX$.
  Projection optimality gives
  \begin{equation*}
    \n_t^\trs (\u_t - \x_t) \le 0.
  \end{equation*}
  If $\n_t = \bm0$, then $\d_t = \g_t$ and the claim is trivial; otherwise,
  \begin{equation*}
    \begin{aligned}
      \d_t^\trs (\y_t - \u_t) - \g_t^\trs (\x_t - \u_t)
      &= \g_t^\trs \n_t + \max\lbr{-\g_t^\trs \n_t, 0} \sbr{1 + \frac{\n_t^\trs (\x_t - \u_t)}{\norme{\n_t}^2}}
      \ge 0.
    \end{aligned}
  \end{equation*}
  If $\g_t^\trs \n_t < 0$, the correction removes the component of $\g_t$ parallel to $\n_t$.
  Thus $\norme{\d_t} \le \norme{\g_t} \le G$.
\end{prf}

\begin{lem}[surrogate regularity]
  \label{lem:surrogate-regularity}
  Under the hypotheses and construction of \cref{thm:pointwise-transfer}, use the applicable parameters in \cref{tab:loss-parameters}.

  On $\YY$, the loss $h_t$ is linear and $G_h$-Lipschitz in the convex mode;
  $\alpha_h$-exp-concave, $G_h$-Lipschitz, and $H_h$-smooth in the exp-concave mode;
  and $\lambda_h$-strongly convex, $G_h$-Lipschitz, and $H_h$-smooth in the strongly convex mode.

  In either curved mode, for every dynamic comparator sequence $\lbr{\u_t}_{t=1}^T\in\XX^T$, define the gradient-shifted centers $\c_t \triangleq \u_t - H_h^{-1} \gr h_t(\u_t)$.
  Then $\c_t\in\YY$.
\end{lem}

\begin{prf}[Proof of \cref{lem:surrogate-regularity}]
  In the convex mode, the linear surrogate has gradient $\d_t$ and zero Hessian.
  By \cref{lem:projection-correction},
  \begin{equation*}
    \norme{\gr h_t(\y)} = \norme{\d_t} \le G = G_h, \qquad
    \normop{\gr^2 h_t(\y)} = 0.
  \end{equation*}

  In the exp-concave mode, $\gr h_t(\y) = [1 + \gamma \d_t^\trs (\y - \y_t)] \d_t$ and $\gr^2 h_t(\y) = \gamma \d_t \d_t^\trs$.
  The definition of $\gamma$ in \cref{eq:main-gamma} gives $\gamma\le1/(2D_\XX G)$, and $G_h>G>0$, so
  \begin{equation*}
    \gamma G^2 \le \frac{G}{2D_\XX} < \frac{2G_h}{D_\XX} = H_h.
  \end{equation*}
  Thus $\max\lbr{\gamma G^2,2G_h/D_\XX}=2G_h/D_\XX$, which justifies the choice of $H_h$.
  Since $|1 + \gamma \d_t^\trs (\y - \y_t)| \le 1 + \gamma D_\YY G$, the definitions give
  \begin{equation*}
    \norme{\gr h_t(\y)} \le G_h, \qquad
    \normop{\gr^2 h_t(\y)} \le H_h.
  \end{equation*}
  The same scalar bound gives $\alpha_h \gr h_t(\y)\gr h_t(\y)^\trs \preceq \gamma\d_t\d_t^\trs=\gr^2h_t(\y)$, proving $\alpha_h$-exp-concavity.
  Moreover, the parameters satisfy $\alpha_h D_\YY G_h = \frac{\gamma D_\YY G}{1 + \gamma D_\YY G} < 1$.
  Hence $\min\lbr{\alpha_h,1/(D_\YY G_h)}=\alpha_h$, and \cref{lem:exp-first-order} applies on $\YY$ with $\gamma_h=\alpha_h/2$.

  In the strongly convex mode, $\gr h_t(\y) = \d_t + \lambda(\y - \x_t)$ and $\gr^2 h_t(\y) = \lambda_h I_d$.
  Here
  \begin{equation*}
    \frac{2G_h}{D_\XX} = \frac{2G}{D_\XX} + \lambda\sbr{1+\frac{D_\YY}{D_\XX}} > \lambda,
  \end{equation*}
  so $\max\lbr{\lambda,2G_h/D_\XX}=2G_h/D_\XX=H_h$.
  Therefore
  \begin{equation*}
    \norme{\gr h_t(\y)} \le G_h, \qquad
    \normop{\gr^2 h_t(\y)} \le H_h.
  \end{equation*}

  Finally, $\norme{\c_t} \le D_\XX / 2 + G_h / H_h = D_\XX = D_\YY/2$ by $H_h=2G_h/D_\XX$.
\end{prf}

\begin{lem}[quadratic exp-concavity support \citep{journals/ml/HazanAK07}]
  \label{lem:exp-first-order}
  Let $\alpha,G,D_\XX>0$.
  Let $f$ be differentiable, $\alpha$-exp-concave, and $G$-Lipschitz on a convex domain in $\reals^d$ of diameter at most $D_\XX$.
  If $0 < \gamma \le \frac{1}{2} \min\{\alpha, 1/(D_\XX G)\}$, then for every pair $(\x,\u)$ in the domain,
  \begin{equation*}
    f(\x) - f(\u)
    \le A - \frac{\gamma}{2} A^2,
    \qquad A \triangleq \gr f(\x)^\trs (\x - \u).
  \end{equation*}
\end{lem}

\begin{prf}[Proof of \cref{lem:exp-first-order}]
  Set $\bar\alpha = 2 \gamma \le \alpha$.
  Since $a \mapsto a^{\bar\alpha / \alpha}$ is increasing and concave on $\reals_+$, the function $\exp(-\bar\alpha f)$ is concave.
  Its tangent inequality at $\x$ gives
  \begin{equation*}
    0
    < \exp(-\bar\alpha f(\u))
    \le \exp(-\bar\alpha f(\x))\sbr{1+\bar\alpha A}.
  \end{equation*}
  Thus $1+\bar\alpha A>0$ and
  $f(\x)-f(\u)\le\bar\alpha^{-1}\log(1+\bar\alpha A)$.
  The Lipschitz and diameter bounds give $|A|\le GD_\XX$, so $|\bar\alpha A|\le1$.
  Applying $\log(1+a)\le a-a^2/4$ for $-1<a\le1$ and using $\bar\alpha=2\gamma$ proves the claim.
\end{prf}

\subsection{Proof of \texorpdfstring{\cref{thm:improper-gap}}{Theorem~\ref{thm:improper-gap}}}
\label{sec:first-order-i2d}

\begin{prf}[Proof of \cref{thm:improper-gap}]
  Set $I=[n]$ and $D_\YY=2D_\XX$, and use the notation $\gap_I$ from \cref{sec:main-hindsight}.

  \paragraph{Part (i): linear losses.}
  If $n = 1$, take $\z_1 = \u_1$, which gives $\gap_I(\z_1) = 0$.
  Suppose henceforth that $n \ge 2$, and set
  \begin{equation*}
    \D_0 = \bm0, \qquad
    \D_t = \sum_{s=1}^t \d_s, \qquad
    m \in \argmax_{1\le t\le n-1} \norme{\D_t}, \qquad
    M = \norme{\D_m}.
  \end{equation*}
  Since $\d_t = \D_t - \D_{t-1}$, summation by parts gives
  \begin{align*}
    \sum_{t=1}^n \d_t^\trs (\u_n - \u_t)
    = \sum_{t=1}^{n-1} (\D_t - \D_{t-1})^\trs (\u_n - \u_t)
    = \sum_{t=1}^{n-1} \D_t^\trs (\u_{t+1}-\u_t).
  \end{align*}
  Therefore
  \begin{equation}
    \label{eq:first-order-linear-abel}
    \abs{\sum_{t=1}^n \d_t^\trs (\u_n - \u_t)}
    \le \sum_{t=1}^{n-1} \norme{\D_t} \norme{\u_{t+1} - \u_t}
    \le M P_I.
  \end{equation}
  Thus every path increment is weighted by a prefix sum of norm at most $M$.

  If $M = 0$, take a single $\z_1 = \u_n$.
  By \cref{eq:first-order-linear-abel}, $\gap_I(\z_1) = 0$.
  If $M > 0$, split at $m$ and shift the first comparator against $\D_m$:
  \begin{equation*}
    I_1 = [1, m], \qquad
    I_2 = [m+1, n], \qquad
    \z_1 = \u_n - \frac{P_I}{M} \D_m, \qquad
    \z_2 = \u_n.
  \end{equation*}
  Since $\sum_{t=1}^m \d_t = \D_m$ and $\norme{\D_m} = M$, the corresponding approximation error satisfies
  \begin{align*}
    \gap_I(\z_1, \z_2)
    = \sum_{t=1}^n \d_t^\trs (\u_n-\u_t) - \frac{P_I}{M} \D_m^\trs \D_m
    \le \abs{\sum_{t=1}^n \d_t^\trs (\u_n-\u_t)} - M P_I
    \le 0,
  \end{align*}
  where the last step is \cref{eq:first-order-linear-abel}.

  In every case, each interval comparator satisfies $\norme{\z_s} \le \norme{\u_n} + P_I \le D_\XX = D_\YY/2$ and hence belongs to $\YY$.

  \paragraph{Part (ii): smooth losses.}
  Let $H_h>0$ be the smoothness coefficient and define the gradient-shifted centers
  \begin{equation*}
    \c_t \triangleq \u_t - H_h^{-1} \gr h_t(\u_t).
  \end{equation*}
  We note that \citet{COLT21:baby-exp-concave,AISTATS22:sc-proper} use similar centers in their proofs.
  By the condition of this theorem, each $\c_t \in \YY$.
  Each interval comparator below is the mean of the centers on its piece, so
  \begin{equation*}
    \norme{\z_s} \le |I_s|^{-1}\sum_{t \in I_s}\norme{\c_t} \le D_\YY/2, \qquad \z_s \in \YY.
  \end{equation*}

  If $n = 1$, take $I_1 = [1]$ and $\z_1 = \c_1$.
  Since $\gr h_1(\u_1) = H_h(\u_1-\c_1)$,
  smoothness gives
  \begin{align*}
    \gap_{[1]}(\z_1)
    = h_1(\c_1) - h_1(\u_1)
    &\le \gr h_1(\u_1)^\trs (\c_1 - \u_1) + \frac{H_h}{2} \norme{\c_1 - \u_1}^2\\
    &= -H_h \norme{\u_1-\c_1}^2 + \frac{H_h}{2} \norme{\u_1-\c_1}^2
    = -\frac{H_h}{2} \norme{\u_1-\c_1}^2
    \le 0.
  \end{align*}
  Suppose henceforth that $n \ge 2$,
  and let
  \begin{equation*}
    \bar\c = \frac{1}{n} \sum_{t=1}^n \c_t, \qquad
    \r_t = \c_t - \bar\c, \qquad
    \R_t = \sum_{s=1}^t \r_s, \qquad
    m \in \argmax_{1\le t\le n-1} \norme{\R_t}, \qquad
    M = \norme{\R_m}.
  \end{equation*}
  Split at $m$ and use the mean center on each piece:
  \begin{equation*}
    I_1 = [1, m], \qquad
    I_2 = [m+1, n], \qquad
    \z_s = \frac{1}{|I_s|} \sum_{t \in I_s} \c_t, \qquad
    \bdelta_s = \z_s - \bar\c.
  \end{equation*}
  By definition,
  \begin{equation*}
    \R_n = \sum_{t=1}^n \r_t = \bm0, \qquad
    \bdelta_1 = \frac{\R_m}{m}, \qquad
    \bdelta_2 = - \frac{\R_m}{n-m}; \qquad
    \sum_{t \in I_s} (\z_s - \c_t) = \bm0.
  \end{equation*}

  For $t \in I_s$, the identity $\gr h_t(\u_t) = H_h(\u_t-\c_t)$ gives
  \begin{align*}
    h_t(\z_s) - h_t(\u_t)
    \le \gr h_t(\u_t)^\trs (\z_s-\u_t) + \frac{H_h}{2} \norme{\z_s-\u_t}^2
    = \frac{H_h}{2} \sbr{ \norme{\z_s-\c_t}^2 - \norme{\u_t-\c_t}^2 }.
  \end{align*}
  Consequently,
  \begin{equation}
    \begin{aligned}
      \label{eq:first-order-smooth-energy}
      \sbr{\frac{H_h}{2}}^{-1} \gap_{[n]}(\z_1, \z_2)
      &\leq \sum_{s=1}^2 \sum_{t \in I_s} \norme{\z_s-\c_t}^2 - \sum_{t=1}^n \norme{\u_t-\c_t}^2 \\
      &= \sbr{ \sum_{t=1}^n \norme{\r_t}^2 - \sum_{s=1}^2 |I_s| \norme{\bdelta_s}^2 } - \sum_{t=1}^n \sbr{ \norme{\u_t-\bar\c}^2 + \norme{\r_t}^2 - 2 \r_t^\trs \u_t } \\
      &= - \sum_{s=1}^2 |I_s| \norme{\bdelta_s}^2 - \sum_{t=1}^n \norme{\u_t-\bar\c}^2 + 2 \sum_{t=1}^n \r_t^\trs \u_t,
    \end{aligned}
  \end{equation}
  where
  \begin{align}
    \sum_{s=1}^2 |I_s| \norme{\bdelta_s}^2
    = \sbr{\frac{1}{m} + \frac{1}{n-m}} M^2
    = \frac{n}{m(n-m)} M^2
    \ge \frac{4}{n} M^2.
    \label{eq:first-order-smooth-mean-energy}
  \end{align}
  With $\R_0 = \R_n = \bm0$ and $\r_t = \R_t - \R_{t-1}$, summation by parts gives
  \begin{align}
    \sum_{t=1}^n \r_t^\trs \u_t
    = \sum_{t=1}^n (\R_t-\R_{t-1})^\trs \u_t
    = \sum_{t=1}^{n-1} \R_t^\trs (\u_t-\u_{t+1})
    \le M P_{[n]}.
    \label{eq:first-order-smooth-abel}
  \end{align}
  Dropping the nonpositive middle term in \cref{eq:first-order-smooth-energy},
  then using \cref{eq:first-order-smooth-mean-energy,eq:first-order-smooth-abel}, gives
  \begin{align*}
    \sbr{\frac{H_h}{2}}^{-1} \gap_{[n]}(\z_1, \z_2)
    \le - \frac{4}{n} M^2 + 2 M P_{[n]}
    = \frac{n}{4} P_{[n]}^2 - \frac{4}{n} \sbr{M - \frac{n}{4} P_{[n]}}^2
    \le \frac{n}{4} P_{[n]}^2.
  \end{align*}
\end{prf}

\subsection{Extending \texorpdfstring{\cref{thm:algorithm-reduction}}{Theorem~\ref{thm:algorithm-reduction}} to interval dynamic regret}
\label{sec:interval-dynamic-proof}

Interval dynamic regret measures regret against a changing comparator on each interval \citep{AISTATS20:Zhang,ICML20:Ashok,JMLR'25:efficient,arXiv26:gradient-variation-interval}.
Our wrapper, \cref{alg:main-reduction}, already guarantees the following rates simultaneously on all intervals.

\begin{cor}[dynamic regret on every interval]
  \label{cor:interval-dynamic}
  Under the assumptions and with the same run as in \cref{thm:algorithm-reduction}, simultaneously for every $I=[a,b]\subseteq [T]$ and every comparator path $\lbr{\u_t}_{t\in I}\in\XX^{|I|}$,
  \begin{equation*}
    \idreg_I(\lbr{\u_t}_{t\in I}) \triangleq \sum_{t\in I} f_t(\x_t) - \sum_{t\in I} f_t(\u_t)
  \end{equation*}
  satisfies the following bounds, where $P_I=\sum_{t=a+1}^b\norme{\u_t-\u_{t-1}}$.
  \begin{center}
    \begin{tabular}{@{}cc@{}}
      \toprule
      Loss class & Upper bound on $\idreg_I(\lbr{\u_t}_{t\in I})$ \\
      \midrule
      Convex & $\Ot\sbr{\sqrt{|I|}+\sqrt{P_I|I|}}$ \\
      Exp-concave & $\Ot\sbr{d+d^{\frac{2}{3}}P_I^{\frac{2}{3}}|I|^{\frac{1}{3}}}$ \\
      Strongly convex & $\Ot\sbr{1+P_I^{\frac{2}{3}}|I|^{\frac{1}{3}}}$ \\
      \bottomrule
    \end{tabular}
  \end{center}
\end{cor}

\begin{prf}[Proof of \cref{cor:interval-dynamic}]
  The proof of \cref{thm:algorithm-reduction} combines \cref{thm:pointwise-transfer,lem:surrogate-regularity,lem:first-order-partition,thm:improper-gap} with the OIR interface in \cref{dfn:first-order-algorithm}.
  These results apply to any interval $I\subseteq[T]$ after reindexing, so the same argument gives the claimed bounds with $T$ and $P_T$ replaced by $|I|$ and $P_I$, respectively.
  The OIR guarantees hold simultaneously for all intervals and interval comparators on the same run.
\end{prf}

\end{document}